\documentclass[conference]{IEEEtran}
\IEEEoverridecommandlockouts
\usepackage{amsmath,amssymb,amsfonts}
\usepackage{algorithm,algpseudocode}
\newtheorem{theorem}{Theorem}
\newtheorem{lemma}{Lemma}

\newtheorem{remark}{Remark}
\newtheorem{assumption}{Assumption}

\newtheorem{problem}{Problem}
\usepackage{graphicx}
\usepackage{textcomp}
\usepackage{xcolor}
\usepackage{makecell}
\usepackage{adjustbox}
\usepackage{subfigure}
\usepackage{multirow}
\usepackage[font=small,aboveskip=2pt,belowskip=-5pt]{caption}
\newcommand{\norm}[1]{\left\lVert #1 \right\rVert}
\newcommand{\R}{\mathbb R}

\usepackage[numbers]{natbib}
\usepackage{multicol}
\begin{document}

\title{Uncertain Pose Estimation during Contact Tasks using Differentiable Contact Features}
\author{\authorblockN{Jeongmin Lee*,
Minji Lee* and Dongjun Lee\thanks{*equal contribution}}
\authorblockA{Department of Mechanical Engineering, IAMD and IOER, Seoul National University}
\authorblockA{Email: \{ljmlgh,mingg8,djlee\}@snu.ac.kr}}

\maketitle

\begin{abstract}
For many robotic manipulation and contact tasks, it is crucial to accurately estimate uncertain object poses, for which certain geometry and sensor information are fused in some optimal fashion. Previous results for this problem primarily adopt sampling-based or end-to-end learning methods, which yet often suffer from the issues of efficiency and generalizability.
In this paper, we propose a novel differentiable framework for this uncertain pose estimation during contact, so that it can be solved in an efficient and accurate manner with gradient-based solver. 
To achieve this, we introduce a new geometric definition that is highly adaptable and capable of providing differentiable contact features. 
Then we approach the problem from a bi-level perspective and utilize the gradient of these contact features along with differentiable optimization to efficiently solve for the uncertain pose.
Several scenarios are implemented to demonstrate how the proposed framework can improve existing methods.
\end{abstract}

\IEEEpeerreviewmaketitle


\section{Introduction}

Contact has always been considered the challenging part of robot manipulation.
Unlike free-space motion, 
contact constraints are complex to model, complicated to numerically solve, and difficult to find an appropriate strategy to handle well.
As a result, learning-based methods have been widely adopted in this field, with many impressive results to date \cite{akkaya2019solving,nagabandi2020deep,ibarz2021train}.
Learning-based methods are essentially sampling-based methods with forward-directed results.
That is, they involve collecting data from the actions, analyzing the results, and learning how to produce the best results.
But they are data-dependent, often produce noisy results, and generalization is difficult.
Some techniques such as domain randomization \cite{tobin2017domain} are often utilized, yet it is deemed still necessary to develop more structured and reliable methods.

From this perspective, the topic of differentiable physics has recently emerged.
By building differentiable formulation, gradient-based methods can replace many of the sampling requirements, improving generalization performance and efficiency.
As a result, these techniques have proven to be useful in a variety of applications, including trajectory optimization \cite{geilinger2020add,howell2022dojo}, policy gradient \cite{xu2021accelerated}, system identification \cite{le2021differentiable} and design optimization \cite{xu2021end}.
However, the use of differentiable modeling in contact-intensive tasks that require responding to uncertain environments, such as robot assembly and placement, has not been well addressed.

In this paper, we present a novel differentiable framework which estimates the uncertain pose during
contact tasks from sensor measurements.
Our framework has a wide range of applications, from simple external impact localization to interactive manipulation such as peg-in-hole assembly.
The main contribution of this paper is: 1) we devise a new geometry representation based on a prescribed support function which guarantees to provide differentiable contact features and their efficient computation algorithm; and 2) an efficient bi-level solution scheme based on differentiable optimization for uncertain pose estimation problem.
The proposed methods are validated against both in simulation and experiment, demonstrating the efficacy of our differentiable framework for contact tasks.

The rest of the paper is organized as follows. 
Previous studies related to our work are reviewed in Sec.~\ref{sec:relatedworks} with some preliminary materials presented in Sec.~\ref{sec:preliminary}.
Sec.~\ref{sec:problem} presents a formal formulation of the uncertain pose estimation problem in contact.
Then in Sec.~\ref{sec:diffcontact}, our novel prescribed support function based geometry model for differentiable contact feature is presented. 
Sec.~\ref{sec:uncertaintyest} describes bi-level solution scheme for the estimation, based on the geometry model provided in Sec.~\ref{sec:diffcontact} and differentiable optimization.
Various implementation scenarios with evaluations are provided in Sec.~\ref{sec:resultandeval}, followed by concluding remarks and discussions in Sec.~\ref{sec:conclusion}.

\begin{figure}[t] 
    \centering
    \includegraphics[width=8.0cm]{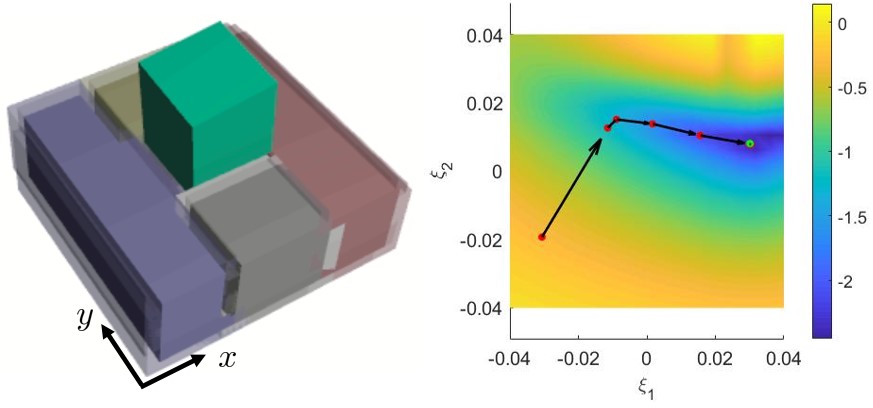}
    \caption{Graphical abstracts illustrating our differentiable pose estimation during contact. Left: A peg-in-hole task performed in a hole with pose uncertainty along the $x$ and $y$ directions. Right: Visualization of the differentiable cost landscape and the gradient-based optimization process utilizing force/torque sensor information acquired through interactions (Green dot: true uncertainty parameter).}
    \label{fig:thumbnail}
\end{figure}

\section{Related Works} \label{sec:relatedworks}

\subsection{Differentiable Contact Formulation}

Many existing studies \cite{geilinger2020add,heiden2021neuralsim,freeman2021brax,howell2022dojo} use collision proxies as simple shapes (point, sphere, plane, etc.). To our best knowledge, attempts to utilize more general geometry have begun to take place very recently.
First, the scope geometry is extended to convex primitives (e.g., cylinder, cone, padded polygon) in \cite{tracy2022differentiable} by utilizing implicit differentiation on conic optimization.
In \cite{turpin2022grasp} and \cite{higuera2022neural}, neural network-based implicit functions such as a signed distance field (SDF) or neural radiance field (NeRF \cite{mildenhall2021nerf}) are used.
However, accurate modeling of contact between the fields is not well-developed and often rely on query point sampling \cite{turpin2022grasp, le2023differentiable}. 
This can lead to reduced applicability and may generate an excessive number of contacts.
In \cite{montaut2022differentiable}, an approach using randomized smoothing with implicit differentiation of Gilbert-Johnson-Keerthi (GJK \cite{gilbert1988fast}) optimality condition is proposed.
However, the gradient may not be consistent with the underlying geometry and may still be myopic.
Instead, we propose to define the object shape through a prescribed support function which provides a direct parametric representation of convex geometry and allows for the exact computation of contact features.
Moreover, theoretical issues on degeneration is addressed, which have not been dealt in previous studies.

\subsection{Uncertainty Handling in Interaction}
Multiple studies have explored the identification of uncertainty in interaction, using a range of sensors.
From visual sensor measurements, 6D pose \cite{deng2020self, xiang2017posecnn} or inertial parameters \cite{mavrakis2020estimation} estimation can be utilized in online during tasks. 
However, vision sensors have limitations in that occlusion can occur, they cannot cover the entire robot body, and are difficult to achieve the high accuracy required for contact-intensive tasks such as peg-in-hole \cite{jin2021contact}.

Therefore, other sensors such as proprioceptive sensor, force/torque (FT) sensor, or more recently vision-based tactile sensors have also widely used.
In many works, encoding sensing measurements for use in manipulation heavily rely on learning-based frameworks.
For example, \cite{lee2019making} combines vision and FT sensor information using self-supervised learning.
In \cite{jin2021contact}, a certain action is performed to acquire FT measurements when contact occurs, and the plotted results are passed through neural network to estimate of the peg pose. 
For tactile sensor, the work in \cite{villalonga2021tactile} estimates the pose of grasped object using neural network and \cite{higuera2022neural} perform tracking of extrinsic contact between object and environment based on neural contact fields.
Similarly, \cite{suresh2022midastouch} performs global localization of the finger and object to a larger object and a long horizon.
These methods are data-dependent and may require re-learning as the use case expands. 
Our work can be combined with these approaches to better exploit the dynamic and kinematic structures, thus improving performance and generalizability.

There also exist some model-based methods to estimate certain information during contact.
In \cite{haddadin2017robot}, a comprehensive survey is provided, but how to deal with object geometry in tandem is rarely addressed.
Studies that address geometry and sensor information together rely primarily on sampling strategies.
For instance, contact particle filter (CPF) \cite{manuelli2016localizing, wang2020contact} presents the way for external contact localization using proprioceptive sensors or force sensors.
Object grasp pose estimation method is also conducted in \cite{sipos2022simultaneous} on the extension of CPF.
Similarly, \cite{von2021precise} presents the Bayseian framework for multi-modal fusion.
These have limitations in that handling multiple contacts is difficult or time-consuming and is inefficient owing to the limitation of the sampling-based method, which only utilizes forward-directed results.
Recently, \cite{ma2021extrinsic} and \cite{kim2021active} develop the optimization based extrinsic contact sensing frameworks using various structured constraints.
In comparison to above works, we aim for a differentiable formulation that can be applied to more general geometric types.

\section{Preliminary} \label{sec:preliminary}

\subsection{Support Function}
For a convex set $\mathcal{C}\subset\mathbb{R}^3$, the support function $h: \mathbb{R}^3\rightarrow \mathbb{R}$ is defined as
\begin{align} \label{eq:supportfunction}
&h(x) = \max_{s(x)\in\mathcal{C}} x^Ts(x)
\end{align}
where $x\in\mathbb{R}^3$ and $s(x)\in\mathcal{C}$ is the farthest point in the $x$ direction among the points in $\mathcal{C}$, called the support point.
Rather than calculating the support function for a given geometry, in this paper, we define the geometry of the object using a prescribed support function.

\subsection{Contact Features} \label{subsec:gdmodel}

\begin{figure}[t] 
    \centering
    \subfigure[Contact features]{
    \includegraphics[width=3.0cm]{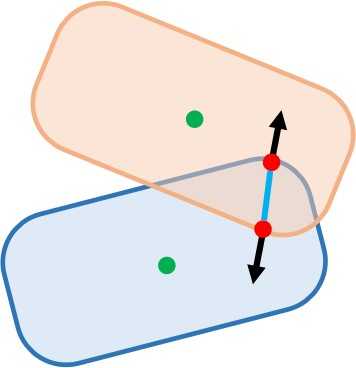}
    \label{fig-growthfeature}
    }
    \subfigure[MTD model]{
    \includegraphics[width=2.1cm]{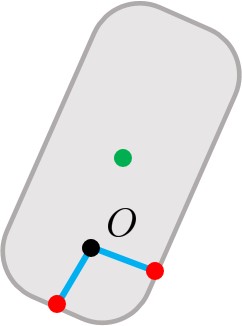}
    \label{fig-mtdmodel}
    }
    \subfigure[GD model]{
    \includegraphics[width=2.1cm]{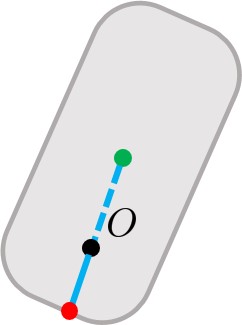}
    \label{fig-gdmodel}
    }
    \caption{Comparison of the minimum translation distance (MTD) model and growth distance (GD) model contact features. Minkowski sum is represented by the gray area. The penetration depth is indicated by the blue segment, the contact witness points by the red points, and the contact normal by the black arrows.}
    \label{fig-convergence}
\end{figure}

The contact features we refer to in this paper are penetration depth, witness points, and contact normal (see Fig.~\ref{fig-growthfeature}).
For general convex shapes, the minimum translation distance (MTD) model is widely adopted \cite{gilbert1988fast,pan2012fcl,schulman2014motion} to define contact features.
The model computes the closest point on the boundary of the Minkowski sum \cite{schneider2014convex} from the origin.
However as depicted in Fig.~\ref{fig-mtdmodel}, the closest point may have a non-unique solution.
The non-uniqueness occurs commonly under deep penetration and sharp geometry.
Although this issue is typically not so critical in simulation because it allows only a small amount of penetration, it is not so for us, as we are aiming for differentiable framework for general manipulation programming, for which such non-uniqueness can pose a serious issue. 

In contrast, the growth distance (GD) model, first proposed in \cite{ong1996growth}, computes the growth factor that two objects ``touch" each other, i.e.,
\begin{align} \label{eq-gdmodel}
\min_\sigma
\text{~s.t.~}\mathcal{C}_1(\sigma) \cap \mathcal{C}_2(\sigma) \neq \emptyset
\end{align}
where $\sigma\in\mathbb{R}^+$ is the growth factor and $\mathcal{C}(\sigma)$ is an increased convex set by the grwoth factor around a given center.
The model was intended to convert contact detection processes from polyhedral objects to linear programming, but we are more interested in the fact that it always guarantees uniqueness of solution \cite{zheng2010fast,tracy2022differentiable}.
This uniqueness can be easily identified using the property that the problem is equivalently substituted by the ray casting problem \cite{zheng2010fast,zheng2013ray} for Minkowski sum (see also Fig.~\ref{fig-gdmodel}).

\subsection{Implicit Function Theorem} \label{subsec:ift}

Consider the multi-variable equation:
\begin{align*}
F(x,y) = 0
\end{align*}
where $x\in\mathbb{R}^{n_x}$, $y\in\mathbb{R}^{n_y}$, and $F:\mathbb{R}^{n_x+n_y}\rightarrow\mathbb{R}^{n_y}$ is the continuously differentiable function.
Then the local solution mapping between $x$ and $y$ is unique and continuously differentiable satisfying
\begin{align} \label{eq:iftheorem}
\frac{dy}{dx} = - \left( \frac{\partial F}{\partial y} \right)^{-1}\frac{\partial F}{\partial x}
\end{align}
if the partial Jacobian $\frac{\partial F}{\partial y}$ is non-singular.
The implicit function theorem enables the use of a function between multiple variables based on an implicit relation.
\section{Problem Formulation} \label{sec:problem}

The main purpose of this paper is to develop the differentiable and general-purposed framework for uncertain pose estimation in interaction.
We define the basic structure of the problem as follows:

\begin{problem} [Uncertain Pose Estimation in Contact]
Given the measurement $\gamma \in \R^{n_\gamma}$, estimate uncertain pose parameter $\xi \in \R^{n_\xi}$ through following optimization problem:
\begin{align} \label{eq:problem_def}
\begin{split}
\min_{\xi,f \in \mathcal C}&~~\frac{1}{2} \norm{\gamma -  \sum_{k=1}^{m(\xi)} P_k(\xi) f_k}_{\Sigma^{-1}}^2  \\
\text{s.t.}&~~g_k(\xi) \ge 0,~~(g_k(\xi))^+ f_k = 0\quad \forall k
\end{split}
\end{align}
where $m$ is the number of collision, $g_k \in \R$, $f_k\in\mathbb{R}^3$, $P_k \in \R^{n_\gamma \times 3}$ are the gap, contact force, and contact mapping matrix (to the measurement) for the $k$-th contact. 
Note that $P_k$ can be expressed as a Jacobian matrix related to the contact witness points and normal.
Also, $\| \cdot \|_{\Sigma^{-1}}^2$ is the Mahalanobis distance defined under the covariance matrix $\Sigma$, $(\cdot)^+=\max(\cdot,0)$, and $\mathcal{C}$ denotes the friction cone set:
\begin{align}
\begin{split}
\mathcal{C} &= \mathcal{C}_1\times\cdots\times\mathcal{C}_{m} \\
\mathcal{C}_k &= \left\{ f_k~|~\mu_k f_{k,n} \ge \| f_{k,t}\| \right\} 
\end{split}
\end{align}
with $\mu,n,t$ being the friction coefficient\footnote{In practice, it is difficult to accurately know the friction coefficient value, so the rough upper value is mainly used.}, subscripts for the normal and tangential direction.
\end{problem}

Here, the measurement $\gamma$ is typically the FT or joint torque sensor value.
It can also be a stack of measurements rather than a single measurement.
Problem 1 can be interpreted as finding the most likely pose and contact force that minimizes the residual of the sensor measurements under several constraints, including the friction cone, non-penetration, and the complementarity constraint that ensures the contact force only acts when the gap is not bigger than zero.

Problem 1 can be seen as a generalization of the problem in \cite{manuelli2016localizing} to deal with the geometry of multiple objects, multiple contact interactions, and various types of uncertainty.  
Moreover, by including additional cost in \eqref{eq:problem_def}, it can be combined with other sensor information (e.g., vision) as well as dynamics condition (see also Sec.~\ref{subsec:additional}).
As a result, it has wide-ranging applications in robotics including grasp pose identification, object tracking, and external impact localization and is easily extensible.
However, there are several challenges to solving a problem: 1) the problem is nonlinear with multiple complementarity constraints, and 2) the differentiability of $m,g,P$ is ambiguous, making it difficult to find a proper gradient direction to optimize.

The following sections describe how to address this problem by making it differentiable. We begin by introducing a geometric representation that enables us to represent $g$ and $P$ in a differentiable manner.

\section{Differentiable Contact Features via Prescribed Support Function} \label{sec:diffcontact}

The computation of differentiable contact features in primitive shapes (e.g., sphere, plane) is simple, but its application is limited.
This section will describe a versatile and efficient scheme based on a prescribed support function for common convex geometry.
The method will later be extended to broader non-convex geometries by using a set of convex geometries.

\subsection{Prescribing Support Function} \label{subsec-geometry}

Following theorem motivates us to model the geometry using a prescribed support function.
\begin{theorem} [\cite{schneider2014convex}]
If $h: \mathbb{R}^3\rightarrow \mathbb{R}$ is a sublinear function that satisfies:
\begin{align*}
\begin{array}{rl}
\text{Positive homogeneity:}&h(\lambda x) = \lambda h(x)~\forall \lambda\ge 0,x\in\mathbb{R}^3 \vspace{1.5mm}\\
\text{Subadditivity:}&h(x+y) \le h(x) + h(y)~\forall x,y\in\mathbb{R}^3 
\end{array}
\end{align*}
then there is a unique convex body corresponding to the support function.
\end{theorem}
This theorem implies the one-to-one relationship between a sublinear function and corresponding convex body.

\begin{figure}[t] 
    \centering
    \subfigure[Proposed]{
    \includegraphics[width=4.1cm]{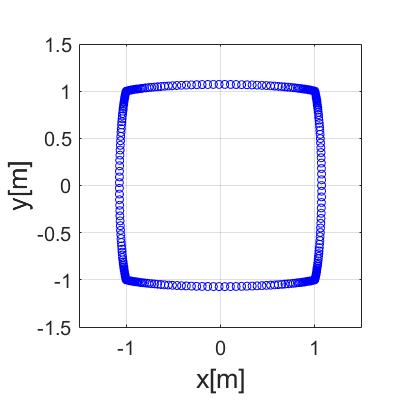}
    }
    \subfigure[Softmax]{
    \includegraphics[width=4.1cm]{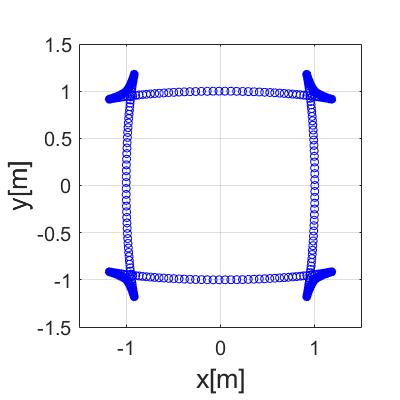}
    }
    \caption{Comparison of geometry obtained by the proposed support function and the naive softmax support function based on exponential. Vertex set is defined as $\left\{[1,1],[1,-1],[-1,1],[-1,-1]\right\}$.}
    \label{fig-maxcompare}
\end{figure}

The question remained is then how to define the prescribed form of the support function.
We first consider the set of vertices i.e., $v_1,\cdots,v_n\in\mathbb{R}^3$.
This vertex set can be determined by the user or obtained from data such as mesh or point cloud.
As it will be generalized under SE(3) transformation in Sec.~\ref{subsec-se3transformation}), here  we assume that the origin is inside the convex hull of the vertices.
Then we can easily find that the support function of the geometry defined as a convex hull is written as
\begin{align} \label{eq-vertexsuport}
&h(x) = \max \left(v_1^Tx,\cdots,v_n^Tx \right )
\end{align}
which is discontinuous.
Instead of the max operator, we use a smoothed version of \eqref{eq-vertexsuport} for differentiable contact feature computation.
The proposed function form is as follows:
\begin{align} \label{eq-blendsupport}
&h(x) = \left( \sum_{i=1}^n \left\{ \max(v_i^Tx,0) \right\}^p \right )^{\frac{1}{p}}
\end{align}
where $p>2$.
Equation \eqref{eq-blendsupport} is similar to the $p$-norm function, but the $\text{abs}(\cdot)$ is replaced by $\max(\cdot,0)$, which naturally culls negative elements.
Then Theorem 2 summarizes an important property of \eqref{eq-blendsupport}.
\begin{theorem}
Given vertex set $v_1,\cdots,v_n$, the function \eqref{eq-blendsupport} is sublinear and twice-differentiable on $\mathbb{R}^3\setminus{\boldsymbol{0}}$.
\end{theorem}
\begin{proof}
Positive homegeneity is trivial. Subadditivity can be shown as
\begin{align*}
h(x)+h(y) 
&=  \left( \sum_{i=1}^n \left\{ (v_i^Tx)^+ \right\} ^p \right )^{\frac{1}{p}}
+ \left( \sum_{i=1}^n \left\{ (v_i^Ty)^+ \right\} ^p \right )^{\frac{1}{p}} \\
&\ge  \left( \sum_{i=1}^n \left\{ (v_i^Tx)^+ + (v_i^Ty)^+ \right\} ^p \right )^{\frac{1}{p}} \\
&\ge  \left( \sum_{i=1}^n \left\{ (v_i^T(x+y))^+ \right\} ^p \right )^{\frac{1}{p}} \\
&= h(x+y)
\end{align*}
using the Minkowski inequality, where $\max(\cdot,0)$ is simplified as $(\cdot)^+$.
Therefore, the function is sublinear.
Twice-differentiablity can be easily verified by using the fact that 
\begin{align*}
&\sum_{i=1}^n \left\{(v_i^Tx)^+ \right\} ^p > 0 
\end{align*}
for $x\in\mathbb{R}^3\setminus{\boldsymbol{0}}$ as the origin is inside the vertex set.
\end{proof}

The properties in Theorem 2 is crucial, as it ensures that any \eqref{eq-blendsupport} always corresponds to some convex geometry -  note from Fig.~\ref{fig-maxcompare} that other classes of support function are not necessarily able to do so.
Fig.~\ref{fig-blendgeometry} depicts various smoothed geometries generated by the support function \eqref{eq-blendsupport}.
We can find that smoothness of the geometry can be easily adjusted using $p$ while retaining convexity and differentiability.

\subsection{Support Point and $\mathrm{SE}(3)$ Transformation} \label{subsec-se3transformation}

\begin{figure}[t] 
    \centering
    \subfigure[Cube]{
    \includegraphics[width=8.0cm]{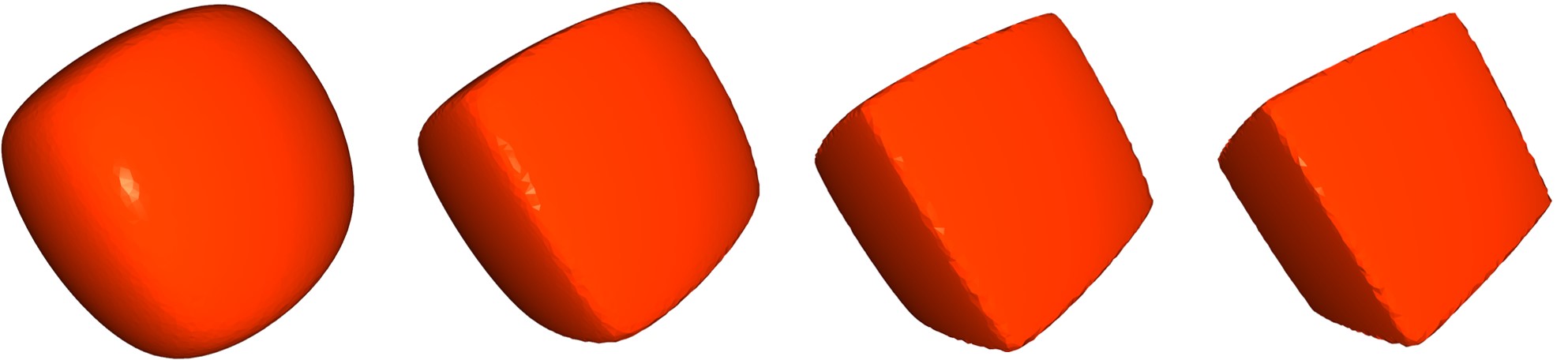}
    }
    \subfigure[Dodecahedron]{
    \includegraphics[width=8.0cm]{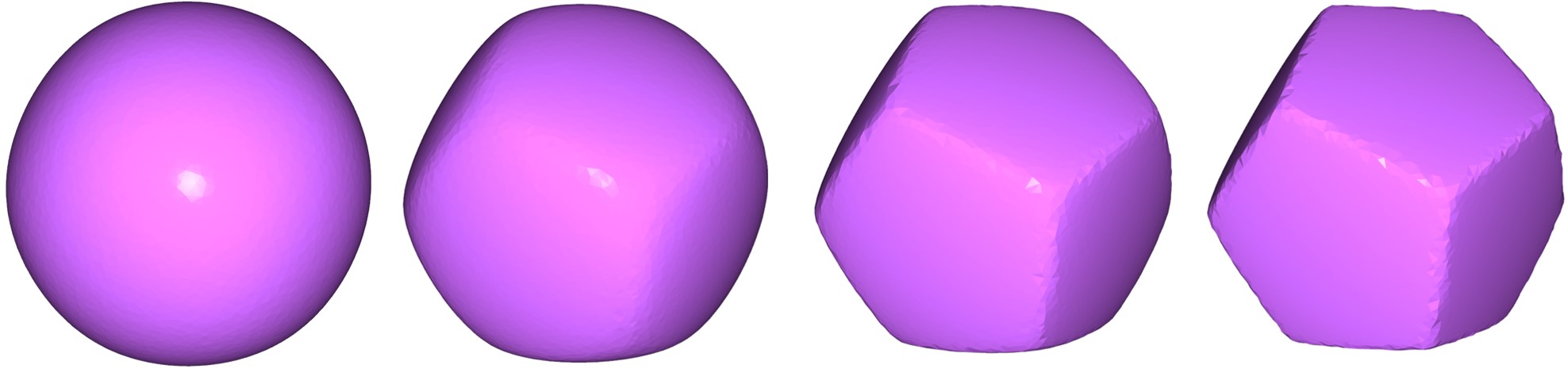}
    }
    \caption{Visualization of geometries represented by the prescribed support function \eqref{eq-blendsupport}. From left to right, $p=5,10,20,40$ are used.}
    \label{fig-blendgeometry}
\end{figure}

From the definition of support function \eqref{eq:supportfunction}, support point $s(x)$ can be derived as follows:
\begin{align} \label{eq-supportpoint}
&s(x) = s(x) + x^T \frac{ds}{dx} = \frac{dh}{dx}
\end{align}
since $x^T\frac{ds}{dx}=0$ holds from the homogeneity. 
Note that the support point can be easily obtained since $h(x)$ in \eqref{eq-blendsupport} is easy to differentiate.
By computing support points \eqref{eq-supportpoint} for various $x$ direction, we can visualize the corresponding shape of geometry.

The support function $h(x)$ and the point $s(x)$ are defined for the geometry that includes the origin. Such geometric representation can be generalized to arbitrary poses through $\mathrm{SE}(3)$ transformations.
Given $h$ and configuration vector $q\in\mathbb{R}^7$ (i.e., position and quaternion), the support function $\bar{h}$ and support point $\bar{s}$ for $q$ and $x$ can be derived as follows:
\begin{align} \label{eq:shbar}
\begin{split}
&\bar{h}(q,x) = h(R(q)^Tx) + p(q)^Tx \\
&\bar{s}(q,x) = R(q)s(R(q)^Tx) + p(q)
\end{split}
\end{align}
where $p(q)\in\mathbb{R}^3$ and $R(q)\in \mathrm{SO}(3)$ are the translation and rotation by $q$.
This transformation \eqref{eq:shbar} is essentially equivalent to converting $x$ to the object local coordinate to obtain $s(R(q)^Tx)$ and then converting it back to the global coordinate.
It can be easily verified that $\bar{f}$ also satisfies the property in Theorem 2, and further twice-differentiable for $q$.

\subsection{Contact Feature Computation}

\begin{figure}[t] 
    \centering
    \subfigure[Active contact]{
    \includegraphics[width=4.1cm]{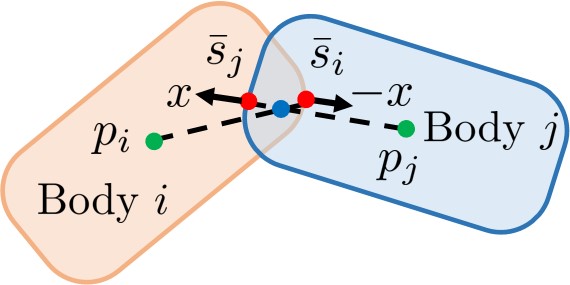}
    }
    \subfigure[Inactive contact]{
    \includegraphics[width=4.1cm]{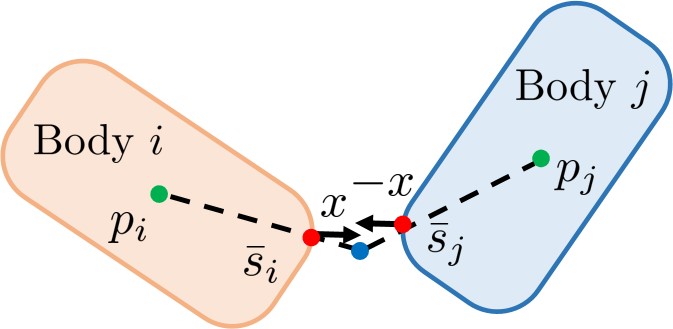}
    }
    \caption{Visualization of the condition in \eqref{eq-growtheq}. Support points (red points) on both bodies extended by the growth factor should meet exactly (blue point).}
    \label{fig-gdsolvervis}
\end{figure}

We compute the contact features based on the GD model described in Sec.~\ref{subsec:gdmodel}.
However as also mentioned in \cite{zheng2010fast}, the methods for functional surfaces rather than discrete geometries are quite limited.
In \cite{zheng2010fast}, using the equivalence of the GD model and ray shooting problem, a method based on the internal expanding procedure is presented.
In \cite{tracy2022differentiable}, optimization \eqref{eq-gdmodel} for convex primitives is formulated via conic optimization and solved using primal-dual interior-point method.
Here, combined with our geometry definition described above, we present an efficient and robust algorithm to solve GD model and its differentiation.
The key concept is to solve the GD model as an unconstrained nonlinear equation by exploiting the support function \eqref{eq-blendsupport}.

\subsubsection{Nonlinear equation}

Our unconstrained formulation employs the solution variables as $x$ (i.e., normal vector for support function input) and growth factor $\sigma\in\mathbb{R}$, resulting in $4$ dimensions.
Then the conditions that the solution must satisfy are: 1) the two support points of each body corresponding to $x$ coincide exactly when extended to $\sigma$; and 2) the normal vector $x$ has unit norm. 
Fig.~\ref{fig-gdsolvervis} visualizes the equivalence of these conditions and the growth distance model, both for active (penetrated) and inactive (separated) contact cases.
The conditions described above can be formulated by the following nonlinear equation: for given bodies $i$ and $j$:
\begin{align} \label{eq-growtheq}
F(x,\sigma,q) = \begin{bmatrix}\sigma (\bar{s}_i-\bar{s}_j) + (1-\sigma)(p_i-p_j) \\
\| x\|^2 - 1
\end{bmatrix}=0
\end{align}
where $\bar{s}_i=\bar{s}_i(x,q_i),\bar{s}_j=\bar{s}_j(-x,q_j)$ and $p=p(q)$. 
The contact detection process is then reduced to solve \eqref{eq-growtheq} with respect to $x,\sigma$ given the configuration $q_i$ and $q_j$.
Note that the formulation is of fixed dimension (i.e., $4$) regardless of the number of vertices used. 
See Appendix~\ref{subsec:appendixunique} for the statements on uniqueness of the solution.

\subsubsection{Newton solver}

\begin{algorithm}[t]
\caption {Contact Feature Solver} \label{alg-trustregion}
\begin{algorithmic}
\State Initialize $x,\sigma$ using IE procedure
\State Compute $F,J$ for initialized value by \eqref{eq-growtheq},\eqref{eq-growthjacobian}
\State Initialize trust region radius $\delta_{tr}$
\While{not converge}
\State Compute Newton step: $\Delta_{gn} = -J^{-1}F$
\State Compute Cauchy step: $\Delta_{ca} = -\beta J^TF$
\State Find dogleg step $\Delta_{dog}$ by $\Delta_{gn}$, $\Delta_{ca}$, and $\delta_{tr}$ \cite{rosen2012incremental}
\State Update $F,J$ under propagated point by $\Delta_{dog}$
\State Update $\delta_{tr}$ \cite{rosen2012incremental}
\State Update $x,\sigma$ using $\Delta_{dog}$ if the step accepted
\EndWhile
\State Compute differentiation by \eqref{eq:cdift} and \eqref{eq:witnessdiff}
\end{algorithmic}
\end{algorithm}

Theorem 2 ensures that $h$ is twice-differentiable everywhere.
Therefore we can always compute the Jacobian of $F$ in \eqref{eq-growtheq} as follows:
\begin{align} \label{eq-growthjacobian}
\begin{split}
J &= \left[\frac{\partial F}{\partial x},\frac{\partial F}{\partial \sigma}\right]=
\begin{bmatrix}
\sigma \left( \frac{d\bar{s}_i}{dx} + \frac{d\bar{s}_j}{dx}\right) & y\\
2x^T & 0
\end{bmatrix} \\
y &= R_is_i(R_i^Tx)-R_js_j(-R_j^Tx)
\end{split}
\end{align}
and \eqref{eq-growthjacobian} can be applied to Newton-type algorithm to solve nonlinear equation in \eqref{eq-growtheq}.
Specifically, we utilize the trust-region-dogleg method \cite{rosen2012incremental} to achieve stable convergence property.
Due to the simple structure of \eqref{eq-blendsupport}, $\frac{d\bar{s}}{dx}$ is also very easy to compute, much like $s$ (see Appendix~\ref{subsec:appendixa} for detailed derivation).
Consequently (and also due to its low-dimensionality), $J$ can be computed and solved in a highly efficient manner.

\subsubsection{Initialization}

Despite the fact that the trust-region-based method ensures the stability of algorithm, determining a good initial point is critical to practical performance.
With a good initialization, the Newton-based iteration is known to have quadratic convergence.
We find that the outcome of the first iteration of the internal expanding (IE) procedure presented in \cite{zheng2010fast} is useful as an initial point.
See Sec.~\ref{subsec:cdbenchmark} for detailed results.

\subsection{Feature Differentiation}

After obtaining the contact features, the differential values can be computed and used to obtain the gradients for $P$ and $g$ from \eqref{eq:problem_def}.
The conciseness of our GD model solver also makes the process of obtaining contact feature differentiation very efficient.
Applying implicit differentiation to the nonlinear equation \eqref{eq-growtheq}, we get
\begin{align} \label{eq:cdift}
&\frac{\partial F^*}{\partial q} + J^* \left[ \frac{dx^*}{dq};\frac{d\sigma^*}{dq} \right] = 0 
\end{align}
where the superscript $*$ denotes the value at the solution.
As $J^*$ is only a $4\times4$ matrix (and its factorization have already been computed in the solver step), we can obtain $\frac{dx^*}{dq}$ and $\frac{d\sigma^*}{dq}$ (i.e., differentiation of contact normal and growth factor) very efficiently.
Moreover, differentiation of witness points are simply computed as
\begin{align} \label{eq:witnessdiff}
&\frac{d\bar{s}_i^*}{dq} = \frac{\partial\bar{s}_i^*}{\partial q}+\frac{\partial\bar{s}_i^*}{\partial x}\frac{dx^*}{dq}, \quad
\frac{d\bar{s}_j^*}{dq} = \frac{\partial\bar{s}_j^*}{\partial q}+\frac{\partial\bar{s}_j^*}{\partial x}\frac{dx^*}{dq}
\end{align}
where $\frac{\partial\bar{s}_i^*}{\partial x}$ and $\frac{\partial\bar{s}_j^*}{\partial x}$ are already available from the solver.
Overall contact feature computation and differentiation procedure is summarized in Alg. 1.

\subsection{Analysis on Degeneration} \label{subsec:degeneration}
\begin{figure}[t] 
    \centering
    \subfigure[Proposed]{
    \includegraphics[width=4.1cm]{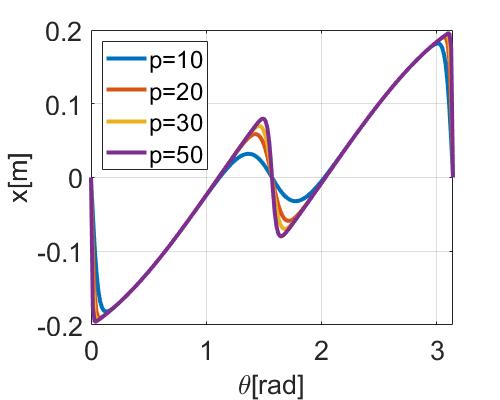}
    }
    \subfigure[Superquadrics]{
    \includegraphics[width=4.1cm]{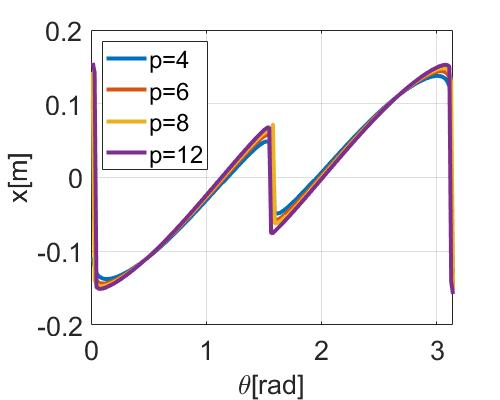}
    }
    \caption{Witness point change plots for two geometric modelings according to rotation angle.}
    \label{fig-degenerate}
\end{figure}
As we can see in \eqref{eq:cdift}, $J^*$ should be non-singular in order to avoid a degenerated situation.
Although the degeneration problem has not been well considered in previous studies, it must be addressed in order to ensure the smooth relation between variables using the implicit function theorem in Sec.~\ref{subsec:ift}.
Without this consideration, pathological cases can arise as demonstrated in \cite{bolte2021nonsmooth}.
In this paper, we theoretically analyze the condition to avoid degeneration for our proposed framework.
We start by making the following assumption.
\begin{assumption}
$\forall x\in\mathbb{R}^3\setminus \boldsymbol{0}$, there exists at least $3$ linearly independent vertices such that $v_i^Tx>0$.
\end{assumption}
This assumption is typically satisfied for shapes that require a sufficient number of vertices to define their geometry, but may not hold for very simple shapes such as a tetrahedron with four vertices. To satisfy the assumption in such cases, additional vertices can be added to the edges of the shape.
Based on this, we present the following lemma:
\begin{lemma}
$\frac{\partial \bar{s}}{\partial x}$ is a positive semi-definite matrix.
Moreover, its rank is $2$ under Assumption 1.
\end{lemma}
\begin{proof}
See Appendix \ref{subsec:appendixb}.
\end{proof}
Based on the lemma, following theorem can be established:
\begin{theorem}
$J^*$ is non-singular under Assumption 1.
\end{theorem}
\begin{proof}
First, $x$ cannot be $0$ at the solution. Now suppose that $J^*$ is singular, therefore for a nonzero vector $z=[z_1,z_2]^T, z_1\in\mathbb{R}^3, z_2\in\mathbb{R}$, $J^*z=0$ holds i.e.,
\begin{align}
&\sigma \left( \frac{d\bar{s}_i}{dx} + \frac{d\bar{s}_j}{dx}\right)z_1+z_2y=0 \label{eq:thm3-1}\\
&x^Tz_1 = 0 \label{eq:thm3-2}
\end{align}
By multiplying $x^T$ to equation \eqref{eq:thm3-1}, we obtain:
\begin{align*}
x^T\left( \sigma \left( \frac{d\bar{s}_i}{dx} + \frac{d\bar{s}_j}{dx}\right)z_1+z_2y \right)=z_2 \left( x^Ty \right)=0
\end{align*}
holds. From the definition of support point, $x^Ty > 0$ holds, therefore we get $z_2=0$.
Now $z_1$ is supposed to be a non-zero vector and perpendicular to $x$ from \eqref{eq:thm3-2}.
Also from the positive semi-definite property in Lemma 1, we have $\frac{d\bar{s}_i}{dx}z_1=0$, which means the row space of $\frac{d\bar{s}_i}{dx}$ must be perpendicular to both $x$ and $z_1$.
Because $x$ and $z_1$ are perpendicular to each other, this contradicts the condition that the rank is $2$.
Therefore $z$ cannot be a non-zero vector, which means $J^*$ is non-singular.
\end{proof}

The theorem provides assurance that a degenerated situation can be avoided, given certain assumptions. This property is generally applicable as the assumptions do not impose significant limitations on its usage.
For demonstration, we conduct a simple experiment that plots the change in the witness point according to the rotation angle of the figure in 2D (see Appendix~\ref{subsec:2dtestsetting} for illustration and details).
As depicted in Fig.~\ref{fig-degenerate}, modeling using superquadrics always induces degeneration (i.e., non-smoothness) even though the parametric equation is smooth and strictly convex.
Our method, on the other hand, is always smooth and exhibits a stiffening pattern as $p$ increases.

\begin{remark}
Non-singular property of $J^*$ is also useful in terms of Newton-based solver (Alg. 1), as it guarantee that the limit point of the sequence satisfies $\|F\|=0$.
\end{remark}
\section{Bi-level Optimization Solver} \label{sec:uncertaintyest}
Combined with geometry modeling described above, in this section, we present the overall gradient-based solution scheme of the estimation problem \eqref{eq:problem_def}.

\subsection{Predefined Number of Contact}

\begin{figure}[t] 
    \centering
    \includegraphics[width=8.4cm]{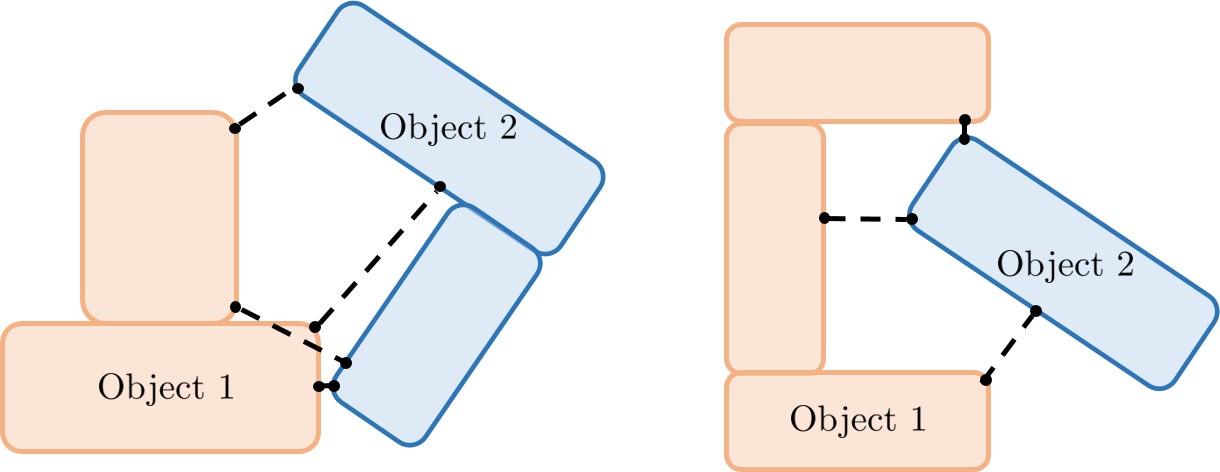}
    \caption{Example of convex decomposition for collisions between two objects. The number of collision $m=4$ for the left and $m=3$ for the right.}
    \label{fig:const_num_of_col}
\end{figure}

From differentiable contact feature suggested in Sec.~\ref{sec:diffcontact}, $P(\xi)$ and $g(\xi)$ are already differentiable. 
Despite this, differentiability of the overall problem is unclear because the number of contacts $m(\xi)$ can change discretely.
We address this issue by keeping the total number of collisions constant.
When two (possibly non-convex) interacting objects are present, we decompose them into $m_1$ and $m_2$ convex geometries, respectively.
Each convex geometry is represented by the method in Sec.~\ref{sec:diffcontact}, therefore only a single collision occurs between them constituting different objects.
Accordingly, we can predefine the collision number as constant, i.e., $m(\xi)=m=m_1 m_2$.
See Fig.~\ref{fig:const_num_of_col} for illustrative examples.
Note that defined contacts are not necessarily active.
We can suppress contact forces for inactive contact by imposing the constraint $(g_k(\xi))^+ f_k=0$.

\subsection{Differentiable Low-level Optimization}
\subsubsection{Smoothing and solving}
For the fixed $\xi$, problem \eqref{eq:problem_def} reduces to find the optimal contact force $f^*$ as
\begin{align*} 
\min_{f \in \mathcal C} \frac{1}{2} \norm{\gamma - P(\xi)f}_{\Sigma^{-1}}^2~\text{s.t.}~(g_k(\xi))^+f_k = 0
\end{align*}
which is a second-order cone programming (SOCP).
Here, the constraint $\left( g_k(\xi) \right)^+ f_k = 0$ is only a $\mathcal{C}^0$ function as it includes $\max$ operator.
For better smoothness, we replace it by a quadratic penalty term in the cost:
\begin{align} \label{eq:socp}
\min_{f \in \mathcal C} {1\over 2} \norm{\gamma - P(\xi)f}_{\Sigma^{-1}}^2 + {k_0 \over 2}\norm{D_{g}^+(\xi) f}^2
\end{align}
where $k_0$ is a penalty coefficient, $P=[P_1,\cdots,P_{m}]$, and $D_{g}^+=\text{blkdiag}(g_1^+I_3,\cdots,g_{m}^+I_3)$.
Then the cost is $\mathcal{C}^1$ function for $\xi$, as each $\left( g_k(\xi) \right)^+$ is squared.
The problem \eqref{eq:socp} is still SOCP and compare to quadratic programming (QP) \cite{amos2017optnet}, it can impose the friction cone without linearization, making it more preferable.
Optimality conditions of \eqref{eq:socp} can be written as
\begin{align} 
&H f + b = J_c^T \lambda \label{eq:pgs-optimality} \\
&0\le\lambda_k \perp c_k \ge 0 \quad \forall k \label{eq:pgs-complementarity} 
\end{align}
where $\perp$ denotes the complementarity, $H\in \R^{3 m_c \times 3 m_c}$, $b\in \R^{3m_c}$, and $c_k\in \R$ are defined as
\begin{align*}
H &= P^T \Sigma^{-1} P + k_0 \left(D_g^+\right)^2 \vspace{0.8mm}\\
b &= -P^T \Sigma^{-1} \gamma  \vspace{0.8mm}\\
c_k &= \mu f_{n,k} - \| f_{t,k} \|  \nonumber 
\end{align*}
where $(\xi)$ is omitted for simplicity, $\lambda =[\lambda_1,\cdots, \lambda_{m}]\in \R^{m}$ is the Lagrange multiplier, and $J_c \in \R^{m \times 3m}$ is the Jacobian $dc\over df$. 
We can see that $c_k$ is non-smooth at $f_k=0$, implying that singularity can occur.
Indeed, the solution of $f_k=0$ is often obtained, particularly in inactive contact.
To relax this issue, we propose to use the following smoothed $c_k$ instead:
\begin{align} \label{eq:smoothck}
c_k = \mu f_{n,k} - \sqrt{f_{t_1,k}^2+f_{t_2,k}^2+\epsilon}
\end{align}
where $\epsilon\in\mathbb{R}^+$ is the small positive value.
Our smoothing scheme has several advantages.
First, as the problem is still strictly convex, its solution set is always singleton.
Also, as \eqref{eq:smoothck} is still analytic, we can resolve the problem efficiently using projection based methods.
Specifically, we utilize the projected Gauss-Seidel (PGS) method \cite{jourdan1998gauss} which is widely used in physics simulation.
In practice, the problem is solved reliably and efficiently as PGS iteration converges to a solution in a small number of iterations.

\subsubsection{Differentiation}
To utilize the gradient method in the high-level optimization, the derivative of the solution of the low-level optimization with respect to the target parameter $\xi$ is required. 
Based on the differentiable contact features in Sec.~\ref{sec:diffcontact}, differentiating \eqref{eq:pgs-optimality} and \eqref{eq:pgs-complementarity} with respect to the parameter $\xi$ is possible, therefore
\begin{align} 
H {df^*\over d\xi} + {dH\over d\xi}f^* + {db\over d\xi} &= J_c^T{d\lambda\over d\xi} + D_\Lambda {df^*\over d\xi} \label{eq:diff_optimality} \\
{d\lambda_k \over d \xi}c_k + \lambda_k {d c_k \over d\xi} &= 0 \quad \forall k \label{eq:diff_complementarity}
\end{align}
can be obtained at the optimal solution $f^*$ of \eqref{eq:socp} where $D_\Lambda = \text{blkdiag}\left(\lambda_1{d^2c_1\over {df^*_1}^2},\cdots,\lambda_m{d^2c_m\over {df^*_m}^2}\right)$.
Here we can classify \eqref{eq:diff_complementarity} into two cases: $c_k=0$ and $c_k > 0$: 
\begin{align*}
\begin{array}{rl}
c_k=0:&\lambda_k {dc_k \over d\xi} = \lambda_k J_{c,k} {d f^*_k \over d \xi} = 0 \vspace{1mm} \\
c_k>0:&{d \lambda_k \over d\xi} = 0 
\end{array}
\end{align*}
This allows us to exclude the components of $\lambda$ that correspond to inactive constraints (i.e., $c_k > 0$) and reduce \eqref{eq:diff_optimality} and \eqref{eq:diff_complementarity} into following form:
\begin{align} \label{eq:lowlevelift}
\begin{bmatrix} H-D_\Lambda & -J_{c, r}^T \\
\Lambda_r J_{c,r} & 0 \end{bmatrix} 
\begin{bmatrix} df^* \over d\xi \\ d\lambda_r \over d \xi\end{bmatrix}
= \begin{bmatrix} -{dH \over d\xi} f^*-{d b\over d\xi} \\ 0 \end{bmatrix}
\end{align}
where $\lambda_r$ and $J_{c,r}$ are the reduced Lagrange multiplier and Jacobian, respectively, and $\Lambda_r$ is a diagonal matrix with the diagonal entries being the elements of $\lambda_r$.
In situations with a positive definite $H$ and no $\lambda_k$ that simultaneously satisfy $\lambda_k=0$ and $c_k=0$, the equation is solvable (See Appendix~\ref{subsec:invertibility} for more details). Otherwise, the least-squares solution can be employed instead \cite{agrawal2019differentiable}.

\subsection{High-level Optimization Solver}

\begin{algorithm}[t] 
\caption {Uncertain Pose Estimation in Contact} \label{alg-uncertainty}
\begin{algorithmic}
\State Initialize $\xi_1,\cdots,\xi_N$ by sampling
\For{$i=1$ to $N$}
\While{not converge}
\State Calculate $f_i^*$ with $\xi_i$ \eqref{eq:socp}
\State Calculate $df_i^* \over d\xi_i$ by solving \eqref{eq:lowlevelift}
\State Calculate the gradient of the cost function of \eqref{eq:high-level}
\State Update $\xi_i$ using Gauss-Newton algorithm
\EndWhile
\EndFor
\State Determine the best $\xi^*$ among $\xi^*_1,\cdots,\xi^*_N$
\end{algorithmic}
\end{algorithm}

By substituting the obtained low-level solution $f^*$ and handling the gap constraint $g_k(\xi) \ge 0$ as penalty functions, we can formulate the high-level problem as
\begin{align} \label{eq:high-level}
\min_\xi {1\over 2} \norm{\gamma - P(\xi) f^*}_{\Sigma^{-1}}^2 + {k_1 \over 2}  \sum_{k=1}^{m}\left( (-g_k(\xi))^+ \right)^2
\end{align}
where $k_1$ is the penalty coefficient to penalize penetration between objects.
As we can obtain the gradient of $f^*$, \eqref{eq:high-level} is now a non-linear least squares problem with differentiable error terms. Hence, we can use off-the-shelf algorithms such as the Gauss-Newton method to solve the problem, which also shows good convergence in practice.

Since the problem \eqref{eq:high-level} is non-convex, there can be multiple local minimum. 
To enhance the ability of our gradient-based algorithm to discover global minimum, we adopt a strategy of sampling the initial pose parameters and selecting the optimal value from among them after optimization.
The overall procedure of our differentiable uncertainty estimation is summarized in Alg.~\ref{alg-uncertainty}.

\subsection{Augmentation} \label{subsec:additional}

The nonlinear least squares problem \eqref{eq:high-level} can be extended by including various additional costs that reflect different aspects of the problem being solved. 
Some examples are as follows:

\subsubsection{Prior}
Prior knowledge of the uncertain pose parameters may be known in many cases.
The following simple Gaussian prior cost can be added in this case:
\begin{align*}
    \frac{1}{2}\| \xi - \xi_{p} \|_{\Sigma_{p}^{-1}}^2
\end{align*}
where $\xi_{p}$ is the prior of $\xi$ and $\Sigma_{p}$ is the covariance matrix.

\subsubsection{Bound constraint}
The bound constraint can be introduced to limit the range of uncertainty. 
In this case, penalty function can be utilized:
\begin{align*}
    \frac{1}{2}\| (-\xi + \xi_{l})^+ \|_{\Sigma_{l}^{-1}}^2 + \frac{1}{2}\| (\xi - \xi_{u})^+ \|_{\Sigma_{u}^{-1}}^2
\end{align*}
where $\xi_{l},\xi_{u}$ are the lower, upper bound of $\xi$ and $\Sigma_{l},\Sigma_{u}$ are the (typically low) covariance matrix.

\subsubsection{Motion model}
The pose parameters can sometimes be estimated over multiple time intervals. 
In such cases, a motion model can be introduced to better estimate the pose parameters. 
See Sec.~\ref{subsec:blindtrack} for an example.
\section{Results and Evaluations} \label{sec:resultandeval}
In this section, various simulation and experiment results are presented to validate the proposed framework.

\subsection{Collision Detection} \label{subsec:cdbenchmark}

We conduct benchmark tests to verify the usefulness of our geometric representations and contact feature computation methods.
We implement the baseline algorithm for GD model as state-of-the-art internal expanding (IE) algorithm \cite{zheng2010fast,zheng2013ray}\footnote{While combination of IE with convex cone projection \cite{zheng2009distance} was also proposed, we find that using IE alone is more suitable for our 3D cases.}, which is proven to be better than GJK based method.
We employ $3$ types of object from YCB dataset (Apple, Mustard, and Sponge, see Appendix~\ref{subsec:ycb} for the images).
Note that our support function based geometric modeling applies to both.
With the residual defined as \eqref{eq-growtheq}, the termination condition is set based on its norm reaching $10^{-10}$.
The performance is recorded under various max iteration number.

\begin{figure}[t] 
    \centering
    \subfigure{
    \includegraphics[width=4.1cm]{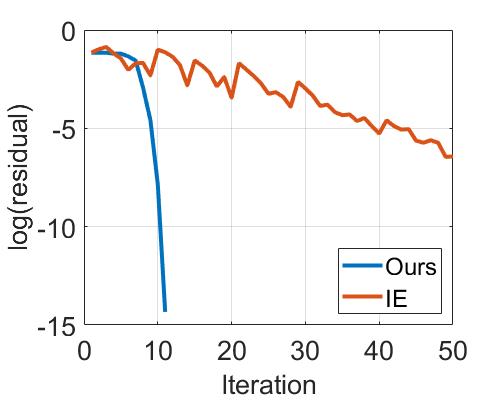}
    }
    \subfigure{
    \includegraphics[width=4.1cm]{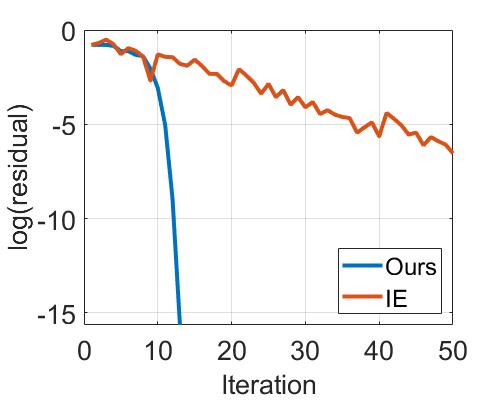}
    }
    \caption{Residual plots for IE and our method for 2 benchmark cases.}
    \label{fig:iecompare}
\end{figure}
\begin{table}[t]
\small
\centering
\renewcommand{\arraystretch}{1.5}{
\resizebox{8cm}{!}{
\begin{tabular}{|c|c|c|c|c|c|c|c|}
\hline
\multicolumn{2}{|c|}{Solver} & \multicolumn{3}{c|}{IE} & \multicolumn{3}{c|}{\textbf{Ours}} \\
\hline
\multicolumn{2}{|c|}{Max Iteration} &
30 & 60 & 90 & 
10 & 15 & 20   \\
\hline
\multirow{2}{*}{A-M} & AT $\downarrow$
& 50.30 & 85.75 & 134.8 
& 23.66 & 27.29 & 29.27 \\
\cline{2-8} & MLR $\uparrow$ 
& 3.905 & 5.587 & 6.271
& 5.262 & 8.109  & 9.954 \\
\hline
\multirow{2}{*}{M-S} & AT $\downarrow$ 
& 32.83 & 65.08 & 85.00
& 15.95 & 22.95 & 24.21  \\
\cline{2-8} & MLR $\uparrow$
& 4.813 & 5.785 & 6.629 
& 4.103 & 7.554 & 9.578     \\ 
\hline 
\multirow{2}{*}{S-A} & AT $\downarrow$ 
& 45.62 & 75.05 & 108.9
& 22.11 & 24.63 & 26.97 \\
\cline{2-8} & MLR $\uparrow$
& 4.962 & 5.645 & 5.998
& 5.099 & 8.043 & 9.729 \\
\hline 
\end{tabular}
}
}   
\caption{Evaluation results for two contact feature (with its differentiation) solvers. 
A, M, S are abbreviations for Apple, Mustard, and Sponge, respectively.
AT: average computation time ($\mu\rm{s}$), MLR: residual converted using $-\log(\cdot)$ before being averaged, therefore bigger is better).}
\label{table:iecompare}
\end{table}

Comparison results in Table~\ref{table:iecompare} demonstrate that our method outperforms the IE algorithm. It achieves faster and more accurate convergence, typically within 20 iterations. Fig.~\ref{fig:iecompare} illustrates the convergence behavior, with our method showing quadratic convergence after a few iterations, while the IE algorithm exhibits first-order convergence. This showcases the advantage of our Newton-type method utilizing the differential value of contact features.

\subsection{External Contact Localization}

\begin{figure}[t]
\centering
\subfigure{
\includegraphics[width=4.1cm]{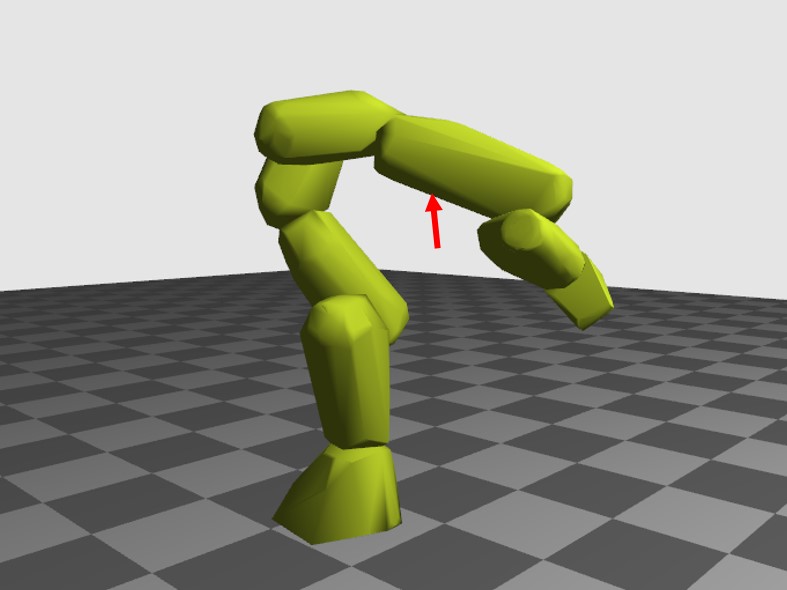}
}
\subfigure{
\includegraphics[width=4.1cm]{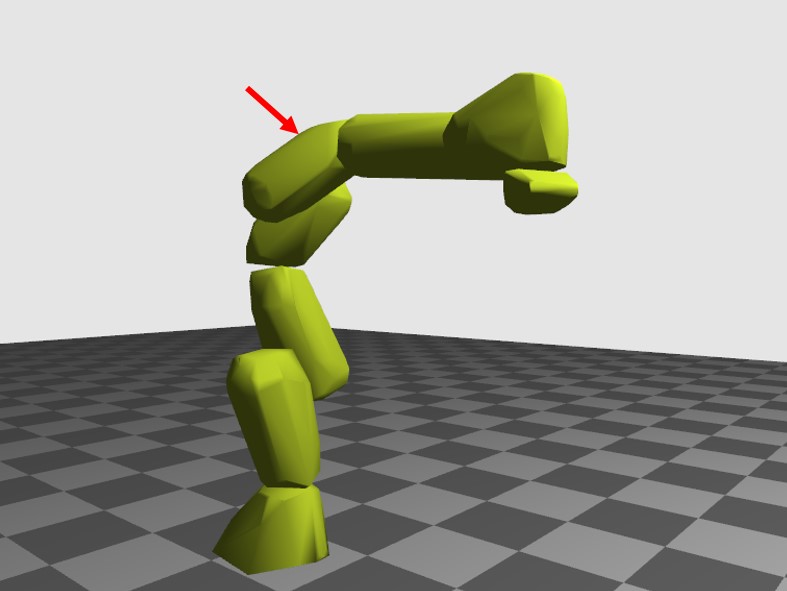}
}
\caption{7-DoF manipulator where each link consists of differentiable collision geometry}
\label{fig:franka}
\end{figure}

\begin{table}[t]
    \centering
    \renewcommand{\arraystretch}{1.4}
    \begin{tabular}{|c | c  | c| c | c| c|c|c|c|c|c|} 
    \hline
    \multicolumn{2}{|c|}{Noise}   & \multicolumn{3}{c|}{Low} & 
    \multicolumn{3}{c|}{High}\\ 
    \hline
    \multicolumn{2}{|c|}{Methods}   & PF & AGD & \textbf{Ours} & PF & AGD & \textbf{Ours}  \\ 
    \hline
    \multirow{3}{*}{Arm 5} 
    & AT $\downarrow$ & 8.94  & 7.86  & 4.10 & 9.00 & 8.03 & 3.95 \\
    & MLE $\uparrow$ & 1.94 &  2.43 & 3.29 & 1.84 & 1.91 & 2.05 \\
    & MLC $\uparrow$ & 2.21  & 3.47 & 5.02 & 1.65 & 2.05 & 2.21 \\
    \hline
    \multirow{3}{*}{Arm 6} 
    & AT $\downarrow$ & 9.46 & 7.64 & 3.95 & 9.38 & 8.04 & 3.93 \\
    & MLE $\uparrow$ & 2.13 & 3.17 & 3.67 & 1.99 & 2.12 & 2.19 \\
    & MLC $\uparrow$ & 1.99 & 3.90 & 5.03 & 1.75 & 2.16 & 2.39 \\
    \hline
    \multirow{3}{*}{Arm 7} 
    & AT $\downarrow$ & 13.1  & 12.0 & 5.27 & 13.2 & 11.7 & 5.39 \\   
    & MLE $\uparrow$ & 2.27  & 3.55 & 4.08 & 2.16 & 2.44 & 2.37 \\
    & MLC $\uparrow$ & 2.05  & 4.33 & 5.29 & 1.62 & 2.15 & 2.18 \\
     \hline 
    \end{tabular}
    \caption{Evaluation results for the external contact localization. AT: average computation time ($\rm{ms}$), MLE/MLC: position error ($\rm{m}$) and cost value converted using $-\log(\cdot)$ before being averaged, therefore bigger is better.}
     \label{tb:external_contact}
\end{table}

External contact localization problem \cite{manuelli2016localizing}, that determines where the contact occurred on the robot arm, is one of the basic examples of Problem 1.
In this case, the uncertain parameter $\xi\in\R^3$ is the collision point, and the measurement $\gamma\in \R^7$ is obtained from the joint torque sensor.
We assume that contact occurs at a single point on a given link\footnote{Here, the contact is point-geometry contact, while the preceding contents mainly describe geometry-geometry contact. However, the problem is still a subset of Problem 1.}, therefore $m=1$.
Existing contact localization algorithms rely heavily on sampling and retraction of points on the mesh, which is computationally expensive.
On these, we verify the efficacy of our differentiable framework here.

For the test cases, Franka Emika Panda \cite{franka} is used, while its links are represented by a convex hull of CAD data, as shown in Fig.~\ref{fig:franka}. 
Two baseline algorithms are employed for comparison in our study. The first algorithm is a particle filter (PF)-based method widely used in the literature \cite{koval2015pose,manuelli2016localizing}. In each iteration, every particle is updated based on the outcome of low-level problems and subsequently projected onto the mesh.
The second baseline algorithm is a more recent approach that utilizes an approximated gradient descent (AGD) combined with low-level problem differentiation \cite{pang2021identifying}. Here the gradient is approximated, as certain terms are disregarded. Additionally, a projection step to the mesh is still required since the derivative of contact features such as the gap and normal vector is unavailable.
Total 1000 trials are conducted and for each trial, a random force is applied to a random position on a link, and the accuracy and computation time are recorded.
For the contact particle filter method, we use 100 particles and iterated for 50 times for convergence. 
For other two methods (AGD and ours) an initial 10 randomly sampled points from the surface of geometry are used as initial guesses. 
Also, for both of the baseline methods, the low-level optimization is performed using the same PGS-based approach employed in our method.

The Table \ref{tb:external_contact} shows the result of external contact localization test, on various links under low/high sensor noise. 
As expected, the particle filter exhibits the lowest performance due to the necessity of conducting the lower-level optimization for each particle and relying on exploration through randomness.
Furthermore, the convergence behavior of AGD is inferior to ours because it is limited to first-order methods with an approximated gradient and necessitates projection. In contrast, our method leverages second-order Gauss-Newton optimization with the exact gradient, leading to improved convergence.

\begin{figure}[t]
\centering
    \centering
    \subfigure[Rectangular peg]{
    \includegraphics[width=8.4cm]{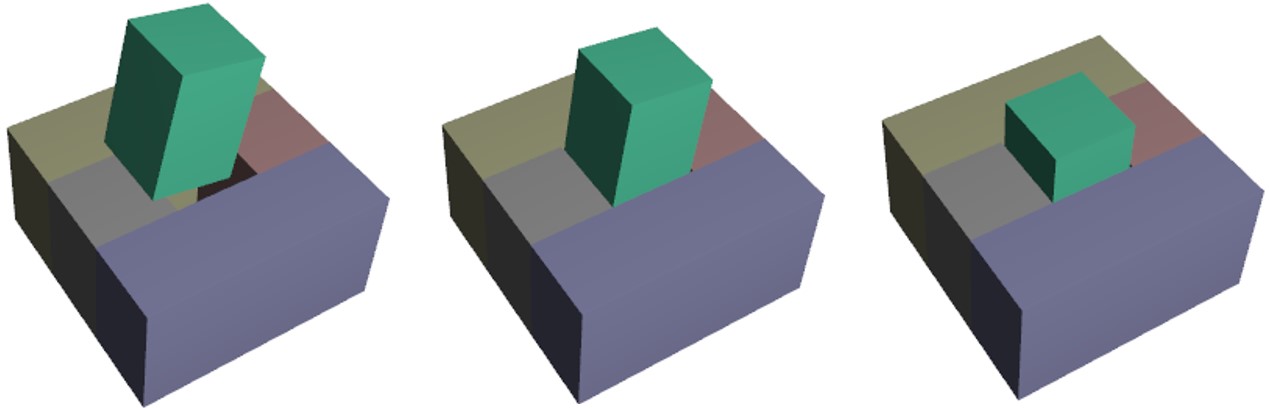}
    }
    \subfigure[Hexagonal peg]{
    \includegraphics[width=8.4CM]{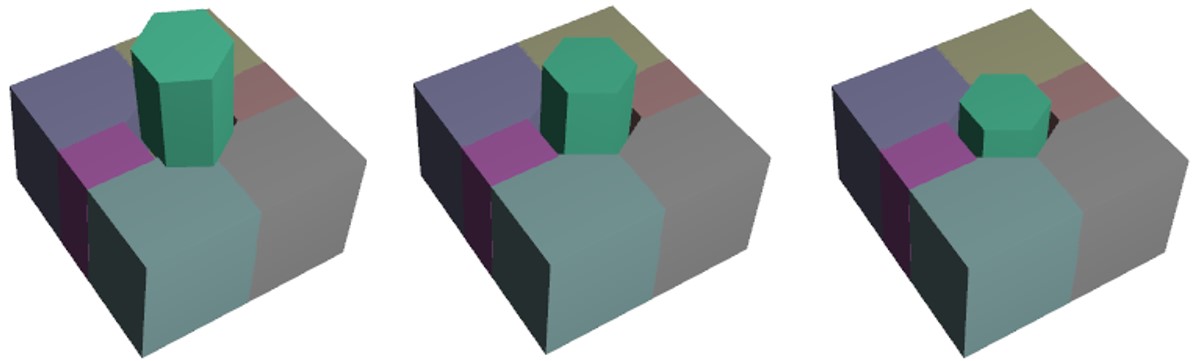}
    }
    \caption{Snapshots of simulation results of peg-in-hole manipulation using our uncertain pose estimation framework in online. Different colors are used to represent convex-decomposed shapes.}
    \label{fig:peginhole_online}
\end{figure}

\subsection{Peg-in-Hole} \label{subsec:peginhole}

Next, the proposed framework is tested on estimating the uncertain grasp pose (i.e., the pose of the peg with respect to the gripper) in peg-in-hole assembly task.
Here the uncertain parameter $\xi \in \R^3$ is the parameterized grasp pose (see Appendix~\ref{subsec:appendixpeg} for details) and the measurement $\gamma\in \R^6$ is from the force/torque sensor on gripper. 
We assume that the gripper and hole poses are known.

The experiment employs two distinct peg geometries: a rectangular prism and a hexagonal prism. 
The rectangular prism has eight vertices, while the hexagonal prism has twelve vertices. 
As shown in Fig.~\ref{fig:peginhole_online}, the hole is decomposed into a total of 4 and 6 convex geometries, and the predefined numbers of collisions $m$ are 4 and 6, respectively.

For the evaluation, we first collect simulation data (FT measurement, ground-truth grasp pose) in a contact situation using the original geometry. 
Here, the data accumulated over three contacts (i.e., $\gamma \in \R^{18}$) is used.
The identification is then performed using the proposed differentiable contact feature, with three initial samples.
For the baseline, we implement the particle filter (PF)-based method similar to \cite{sipos2022simultaneous}.
The PF solves the high-level problem by using the grasp pose as particles with sampling strategy. 
For the low-level problem for each particle, we take the same methodology of our framework for better performance.
Also, the number of particles is 25 (PF25) and 50 (PF50).

\begin{table}[t]
    \centering
    \renewcommand{\arraystretch}{1.2}
    \begin{tabular}{|c|c|c|c|c|c|c|c|} 
     \hline
    \multicolumn{2}{|c|}{Noise} &  \multicolumn{3}{c|}{Low} & \multicolumn{3}{c|}{High} \\ \hline
    \multicolumn{2}{|c|}{Methods} &  \textbf{Ours} & PF25 & PF50 & \textbf{Ours} & PF25 & PF50\\ 
     \hline
       \multirow{4}{*}{Rect} 
      & AT $\downarrow$ & 9.22 & 44.1 & 89.1 & 10.7 & 43.7 & 89.3\\
      & MLPE $\uparrow$ & 4.54 & 1.91 & 2.08 & 3.34 & 1.99 & 2.22\\
      & MLRE $\uparrow$ & 3.72 & 1.06 & 0.94 & 2.47 & 0.83 & 1.16\\
      & MLC $\uparrow$ & 5.74 & -0.912 & -0.66& 2.03 & -0.70 & -0.19\\
     \hline      
     \multirow{4}{*}{Hexa} 
     & AT $\downarrow$ & 18.6 & 100 & 197  & 16.6 & 99.4 & 192\\
     & MLPE $\uparrow$ & 3.71 & 2.21 & 2.37 & 3.87 & 2.47 & 2.34\\
     & MLRE $\uparrow$ & 2.58 & 0.87 & 1.10 & 2.50 & 1.06 & 1.07\\
     & MLC $\uparrow$ & 3.28 & -0.59 & 0.36 & 1.66 & 0.46 & 0.19\\
      \hline
    \end{tabular}
    \caption{Evaluation results for the peg-in-hole assembly task. AT: average computation time ($\rm{ms}$). MLPE/MLRE/MLC: position error ($\rm{m}$), rotation error ($\rm{rad}$) and cost value converted using $-\log(\cdot)$ before being averaged, therefore bigger is better.}
     \label{tb:penginhole_compare}
\end{table}

\begin{table}[t]
    \centering
    \renewcommand{\arraystretch}{1.2}
    \begin{tabular}{|c|c|c|c|c |c|c|c|} 
     \hline
     \multicolumn{2}{|c|}{Noise}& \multicolumn{3}{c|}{Low} & \multicolumn{3}{c|}{High}\\ \hline 
     \multicolumn{2}{|c|}{$p$} & 50 & 60 & 70 & 50 & 60 & 70 \\ 
     \hline
     \multirow{4}{*}{Rect}
      & AT $\downarrow$ & 8.98 & 9.10 & 9.22 &  9.71 & 9.31 & 10.7 \\
      & MLPE $\uparrow$ & 2.92 & 3.27 & 4.54 & 2.75 & 3.02 & 3.34 \\
      & MLRE $\uparrow$ & 2.11 & 2.48 & 3.72 & 1.91 & 2.29 & 2.47 \\
      & MLC $\uparrow$ & 3.43 & 4.32 & 5.74 & 2.03 & 1.95 & 2.03 \\
     \hline
     \multirow{4}{*}{Hexa} 
     & AT $\downarrow$ & 15.5 & 18.2 & 18.6 & 16.5 & 16.1 & 16.6 \\
     & MLPE $\uparrow$ & 2.94 & 3.19 & 3.71 & 3.00 & 3.15 & 3.87 \\
     & MLRE $\uparrow$ & 1.97 & 2.13 & 2.58 & 2.06 & 2.03 & 2.50 \\
     & MLC $\uparrow$ & 1.75 & 2.11 & 3.28 & 1.32 & 1.24 & 1.66\\
     \hline
    \end{tabular}
    \caption{Evaluation results under various smoothing parameters. AT: average computation time ($\rm{ms}$). MLPE/MLRE/MLC: position error ($\rm{m}$), rotation error ($\rm{rad}$) and cost value converted using $-\log(\cdot)$ before being averaged, therefore bigger is better.}
     \label{tb:penginhole_result}
\end{table}

The comparison results are summarized in Table~\ref{tb:penginhole_compare}. A total of ten datasets and two different amounts of noise (standard deviations of $0.1$ and $0.001$) are used. The results clearly demonstrate that the proposed method outperforms the particle filter-based method in terms of accuracy and efficiency. This highlights how the Gauss-Newton algorithm, utilizing gradients, enables rapid convergence to a solution with non-penetration and proper normal/witness points.

Furthermore, a comparative study is conducted by varying the smoothing parameters within our framework. Specifically, the smoothing parameter $p$ is varied from $50$ to $70$. The results are presented in Table~\ref{tb:penginhole_result}. It can be observed that lower values of $p$ result in slightly shorter average computation times. Conversely, higher values of $p$ yield more accurate results as they are closer to the original geometry. Consequently, future investigations could focus on finding fast approximated solutions through proper $p$-smoothing and refining them towards the exact geometry under higher values of $p$.

Additionally, the estimation process can be performed online, involving repeated trials and data augmentation until the task is completed. Simulation snapshots of the peg-in-hole assembly with online estimation are depicted in Fig.~\ref{fig:peginhole_online}. For the video, please refer to our supplementary material. Further visualization results can be found in the Appendix~\ref{subsec:appendixpeg}.

\subsection{Augmentation: Blind Object Tracking} \label{subsec:blindtrack}

\begin{figure}[t] 
    \centering
    \subfigure[Without motion model]{
    \includegraphics[width=4.1cm]{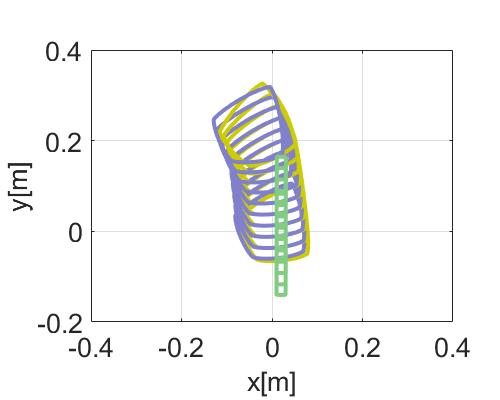}
    }
    \subfigure[With motion model]{
    \includegraphics[width=4.1cm]{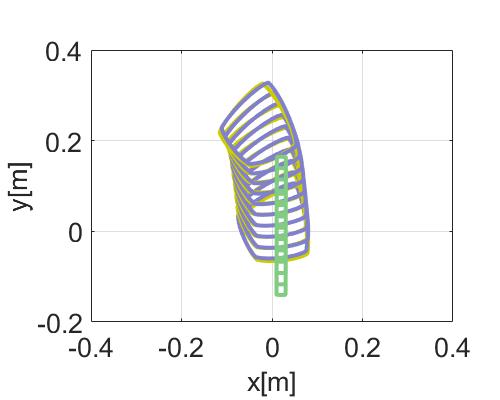}
    }
    \caption{Comparison result of the blind object tracking performance with/without motion model. Green: End effector. Yellow: ground-truth. Purple: Estimation result.}
    \label{fig:pushcompare}
\end{figure}

This subsection provides an example of augmenting our framework with other models, as explained in Section~\ref{subsec:additional}. Specifically, we focus on blind object tracking, which involves tracking an object without relying on visual information during the task. This capability proves beneficial in cluttered environments or areas with limited lighting.
To demonstrate blind object tracking, we configure a pushing environment where the interacting objects are represented by convex geometries based on four vertices. It is worth noting that previous studies \cite{yu2015shape, suresh2021tactile} have tackled similar tasks; however, many of them simplified the shape of the tip to a point. In contrast, our framework allows for a more versatile geometric representation, enabling its applicability to a broader range of end effectors and object shapes. However, still diverse real-world scenarios are remained for future research.

The uncertain parameter $\xi$ and the FT measurement $\gamma$ are stack of values for multiple time intervals.
Here, we adopt the quasi-static motion model based on limit surface \cite{lynch1992manipulation,suresh2021tactile} for augmented cost.
Note that all components in the model is a function of $\xi$ and $f^*$ therefore can be efficiently differentiated.
Ground truth data is obtained from the simulation environment and compared to the estimated results.
The results are illustrated in Fig.~\ref{fig:pushcompare}.
The result from our vanilla cost formulation \eqref{eq:high-level} without motion model exhibits a noticeable bias error. Conversely, when incorporating the motion model, the results demonstrate a substantial improvement in accuracy (reducing the RMSE by $30\%$).

\subsection{Real World Experiment: Dish Placing}

\begin{figure}[t]
\centering
\subfigure{
\includegraphics[width=8.5cm]{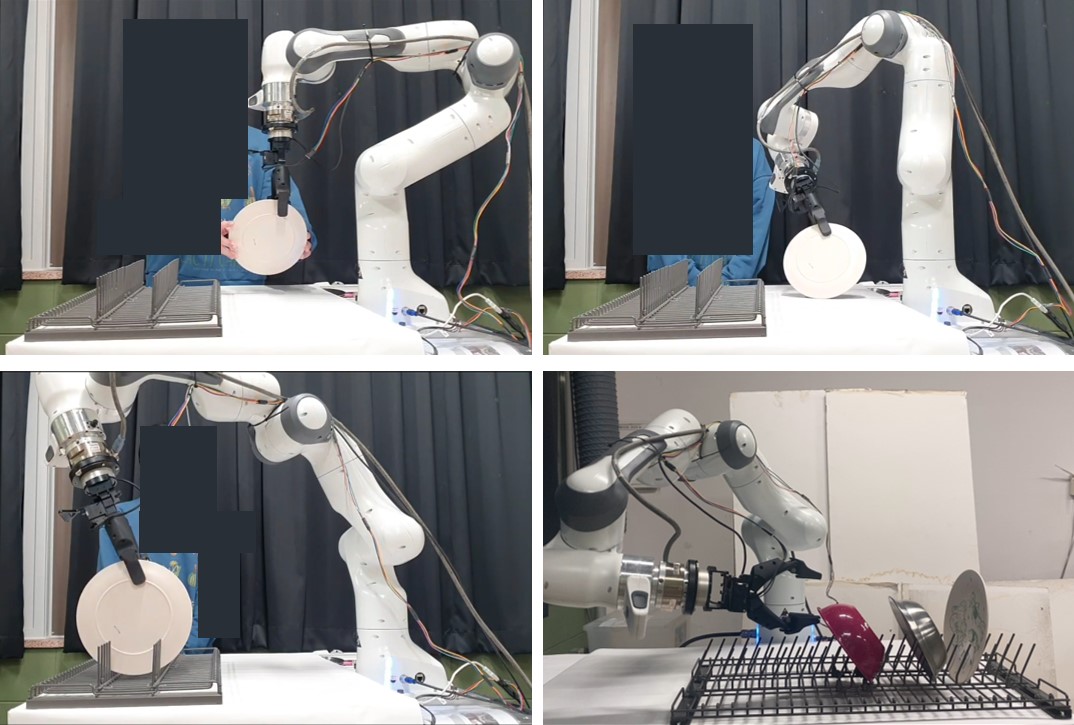}}
\caption{Experimental demonstration of our framework in dish placing task. Top left: A human gives an arbitrary grasp pose. Top right: The robot estimates the uncertainty through interaction with the ground. Bottom left: Placing succeeded by proper estimation. Bottom right: Three dishes are successfully placed in a row.}
\label{fig:dishplacing}
\end{figure}

We deploy our framework in a dish placement task for experimental validation in the real world.
The manipulator is built with Franka Emika Panda and a parallel gripper, and ATI Gamma is utilized as the FT sensor.
Three different dishes are used, with a narrow-spaced dish rack.
Test is conducted as follow: a human makes the gripper to grasp the dish in an arbitrary pose, and the robot identifies the uncertain grasp pose through interaction with the ground.
In the identification process, Alg.~\ref{alg-uncertainty} is employed while the uncertain grasp parameter is modeled in $3$-dimension and the dishes are represented by a smoothed convex hull with a prescribed support function.
Following the identification, the placing is carried out by following the pre-planned trajectory.
If the grasp pose is not estimated correctly, the placement will fail with a stuck or jamming.
Our framework is successfully applied to enable successful performance of dish placement tasks - see Fig.~\ref{fig:dishplacing} for experiment snapshots.
See also our supplementary materials for video and more details.

\section{Discussions and Conclusion} \label{sec:conclusion}

In this paper, we propose a novel uncertain pose estimation framework for interactive robot tasks.
Essentially, we frame the problem as bi-level optimization and devise a way to solve it based on gradient.
Prescribed support function based geometry definition is first presented to make it possible to express differentiable contact features.
The definition also comes with an effective solver algorithm and has useful theoretical properties.
Then by using the predefined number of contacts and differentiating low-level problems, the original problem is finally transformed into a non-linear least squares problem, which can be solved efficiently using conventional gradient-based methods.
Several scenarios are implemented and demonstrate how well our method can outperform currently used sampling-based approaches.

There exists several possible directions for future works.
First, our method is mainly to utilize FT or joint torque sensor information, so combination with more diverse sensors will be useful.
Specifically, embedding the differentiable nonlinear least square derived in our work to general factor graph optimization form will be an important task.
It would also be meaningful to develop a way to handle situations that uncertainty exists in geometry parameters as well as poses.
In a similar vein, specific methodologies for extracting prescribed support function from visual information will be an important topic.
Finally, since our method essentially consists of model-based optimization, it will be interesting to combine it with learning-based methods by modeling it as a single layer \cite{amos2017optnet}.

\section*{Acknowledgement}
This research was supported by Samsung Research and the RS-2022-00144468 of the National Research Foundation (NRF) funded by the Ministry of Science and ICT (MSIT) of Korea. 

\bibliographystyle{plainnat}
\bibliography{reference}

\clearpage

\appendix

\subsection{Additional visualizations}

Fig.~\ref{fig:advis} shows additional visualizations of equations presented in the main contents.

\subsection{On the uniqueness of the solution of \eqref{eq-growtheq}} \label{subsec:appendixunique} 

There are two possible solutions for \eqref{eq-growtheq}, one with a positive $\sigma$ and one with a negative $\sigma$. 
In order to ensure appropriate collision detection, the constraint $\sigma > 0$ is necessary. 
To impose $\sigma>0$, a few simple but additional steps are required in the Newton step, but we find that those are not really necessary under proper initialization of $\sigma$ (in our cases, via IE process).
Thus, we take an unconstrained approach to the problem.

\subsection{Derivation of $\frac{d\bar{s}}{dx}$} \label{subsec:appendixa}

We will derive $\frac{ds}{dx}$, as $\frac{d\bar{s}}{dx}$ is straighforwardly obtained from it.
To simplify notation, let us define $\hat{a}_k=[a_1^k,\cdots,a_n^k]^T$ with $a_i=(v_i^Tx)^+$ and $\tilde{a}_p=\sum\hat{a}_p$.
Then we have
\begin{align*} 
&s(x) = V(\tilde{a}_p)^{\frac{1}{p}-1}\hat{a}_{p-1} \\
&\frac{ds}{dx} = (p-1)V\underbrace{\left( (\tilde{a}_p)^{\frac{1}{p}-1}\text{diag}\left(\hat{a}_{p-2}\right) - (\tilde{a}_p)^{\frac{1}{p}-2}\hat{a}_{p-1}\hat{a}_{p-1}^T \right)}_{A} V^T 
\end{align*}
where $\text{diag}(\cdot)$ denotes the diagonal matrix from a given vector.
Note that the actual computation flow computes the $3\times3$ matrix after computing the $3\times1$ vector $V\hat{a}_{p-1}$, so the complexity is $\mathcal{O}(n)$.

\subsection{Proof of Lemma 1} \label{subsec:appendixb}

Let us first prove the positive semi-definite property.
It is sufficient to show the positive semi-definite property of $A$.
Consider a $n$-dimensional vector $u=[u_1,\cdots,u_n]^T$.
Then
\begin{align*}
u^TAu = \tilde{a}_pu^T\text{diag}\left(\hat{a}_{p-2}\right)u - (u^T\hat{a}_{p-1})^2
\end{align*}
holds. As Cauchy-Schwarz inequality indicates
\begin{align*}
&(a_1^p+\cdots+a_n^p)(a_1^{p-2}u_1^2+\cdots+a_n^{p-2}u_n^2) \\
&\ge (a_1^{p-1}u_1+\cdots+a_n^{p-1}u_n)^2
\end{align*}
it can be confirmed that $u^TAu\ge0$, which means $A$ is positive semi-definite.
Now let us show the rank property.
It is well known that $Au=0$ holds if and only if $u^TAu=0$, if $A$ is a positive semi-definite matrix.
Then as $u=V^Tx$ holds the equality condition of Cauchy-Schwarz inequality, rank of $\frac{ds}{dx}$ is lower than $2$.
Now suppose that there exists $x'$ such that $u'=V^Tx'$ meets the equality condition.
From the assumption, at least three components of $\hat{a}_1$ are non-zero. 
Without loss of generality, let us consider $a_1,a_2,a_3$ are non-zero.
Then $V_{nz}^Tx$ and $V_{nz}^Tx'$ must be parallel, with $V_{nz}=[v_1,v_2,v_3]$.
Finally, as $V_{nz}$ is full rank from the assumption, $x'$ is parallel to $x$, and we conclude that the rank of $\frac{ds}{dx}$ is always $2$.

\subsection{Details on Degeneration Test in Sec.~\ref{subsec:degeneration}} \label{subsec:2dtestsetting}
Illustrations for the degeneration test conducted in Sec.~\ref{subsec:degeneration} is visualized in Fig.~\ref{fig:2dtest}.
Each rectangle shape is represented by $4$ vertices in our geometry model.
Superquadric model can be written as following equation:
\begin{align*}
    \left(\frac{x}{\alpha_1}\right)^p+\left(\frac{y}{\alpha_2}\right)^p=1
\end{align*}
where $p\in\mathbb{R}^+$ is the smoothing parameter similarly to in \eqref{eq-blendsupport}, and $\alpha_1,\alpha_2\in\mathbb{R}$ are the size parameters.  

\subsection{Invertibility of \eqref{eq:lowlevelift}} \label{subsec:invertibility}
Jacobian and Hessian of $c_k$ can be written as
\begin{align*}
{dc_k^*\over df_k} &= \begin{bmatrix}
-{2f_{t_1}\over(f_{t_1}^2+f_{t_2}^2+\epsilon)^{\frac{1}{2}}}&-{2f_{t_2}\over(f_{t_1}^2+f_{t_2}^2+\epsilon)^{\frac{1}{2}}} & \mu 
\end{bmatrix}\\
{d^2c_k^*\over df_k^2} &= 
-{2\over(f_{t_1}^2+f_{t_2}^2+\epsilon)^{\frac{3}{2}}}\begin{bmatrix}
f_{t_2}^2 & f_{t_1}f_{t_2} & 0\\
 f_{t_1}f_{t_2} & f_{t_1}^2 & 0\\
 0& 0 & 0
\end{bmatrix}
\end{align*}
We can find that Jacobian is always rank $1$, as $f_k$ cannot be $0$ to satisfy $c_k\ge0$.
Also, Hessian is always negative semi-definite.
We also know that $\lambda_k\ge 0$, so $D_\Lambda$ is negative semi-definite.
Thus, if $H$ is positive definite and $\lambda_k > 0$, then $H-D_\Lambda$ is also positive definite and invertible. 
Finally, Theorem 2.1 in \cite{dyn1983numerical} concludes the invertibility of the problem.

\subsection{Geometries for Collision Detection Test} \label{subsec:ycb}

Fig.~\ref{fig:ycbobject} shows the objects utilized in the collision detection test conducted in Sec.~\ref{subsec:cdbenchmark}.

\begin{figure}[t] 
    \centering
    \subfigure[Visualization of \eqref{eq-vertexsuport}]{
    \includegraphics[width=3.0cm]{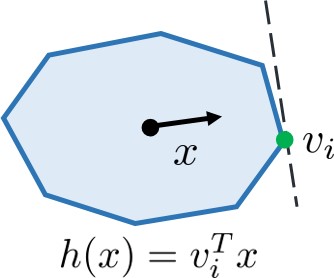}
    }
    \subfigure[Visualization of \eqref{eq:shbar}]{
    \includegraphics[width=5.0cm]{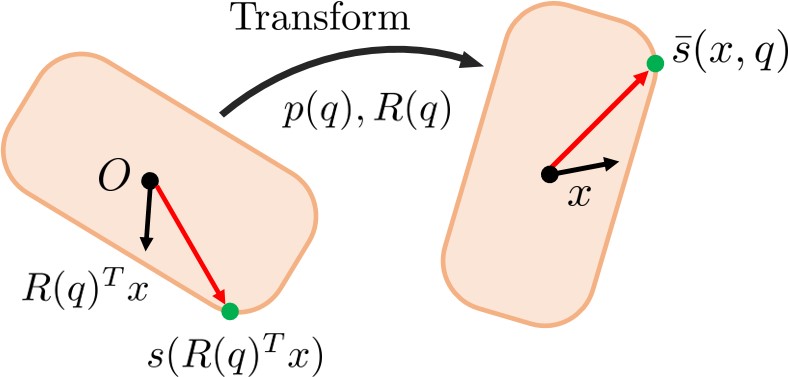}
    }
    \caption{Visualizations of equations. Left: Support function and point for a vertex set in \eqref{eq-vertexsuport}. Right: Support point for SE(3) transformation of body in \eqref{eq:shbar}.}
    \label{fig:advis}
\end{figure}

\begin{figure}[t] 
    \centering
    \subfigure{
    \includegraphics[width=4.1cm]{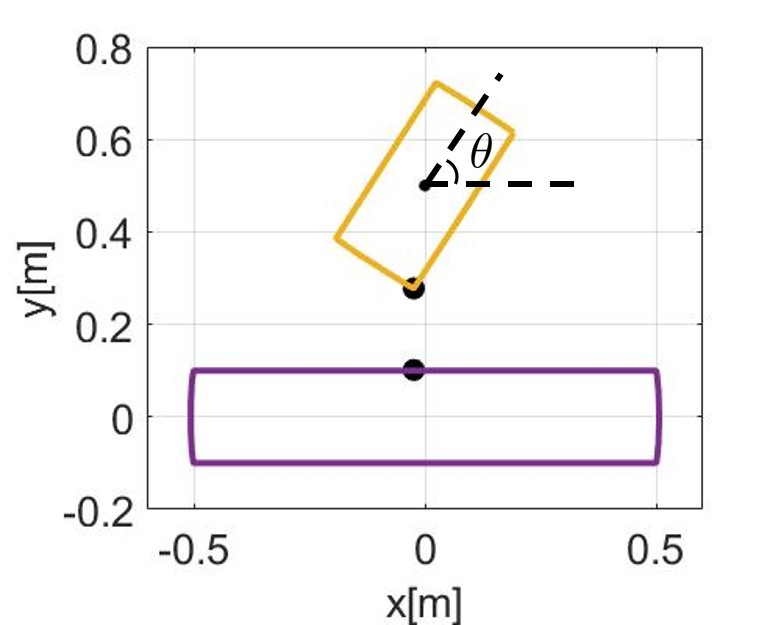}
    }
    \subfigure{
    \includegraphics[width=4.1cm]{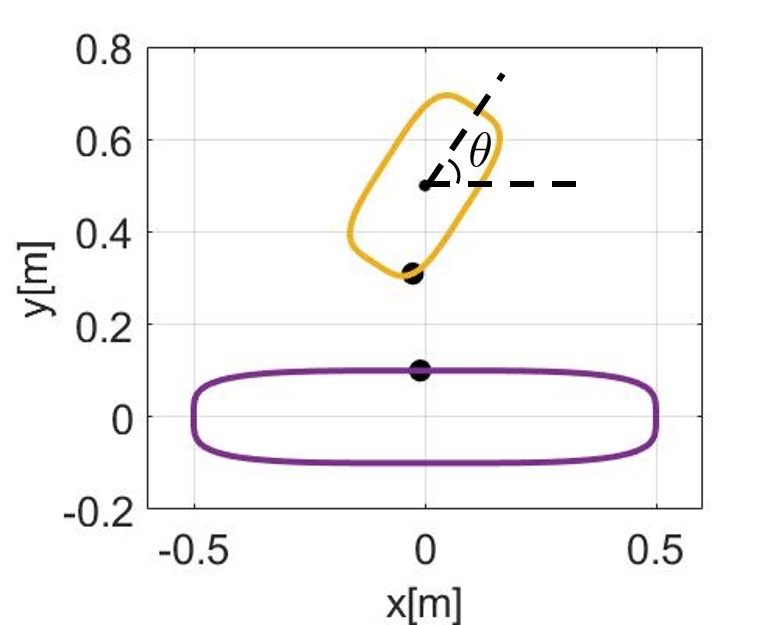}
    }
    \caption{Illustrations for the degeneration test in Sec.~\ref{subsec:degeneration}. Witness points (block dots) are recorded as the rotation angle $\theta$ changes. Left: our support function based modeling. Right: Superquadrics.}
    \label{fig:2dtest}
\end{figure}

\begin{figure}[t]
\centering
\subfigure[Apple]{
\includegraphics[width=2.2cm]{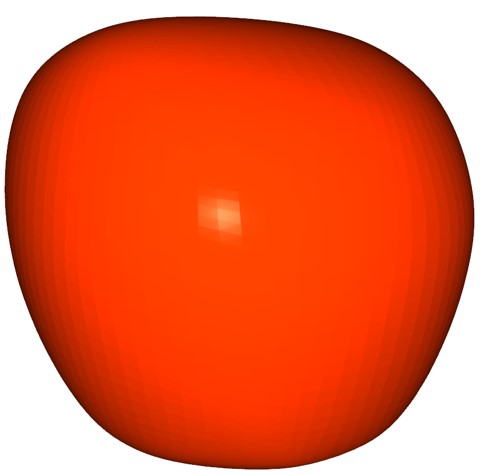}}
\subfigure[Mustard]{
\includegraphics[width=2.0cm]{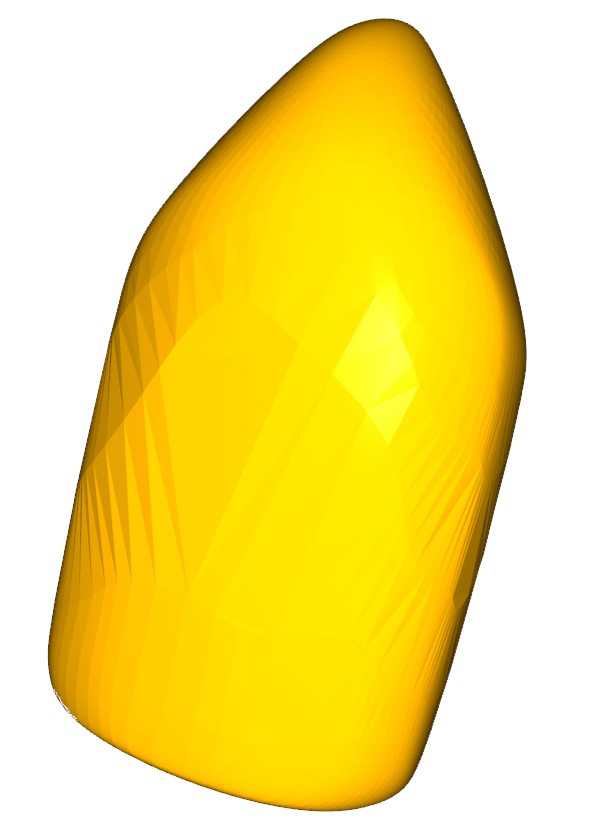}}
\subfigure[Sponge]{
\includegraphics[width=2.2cm]{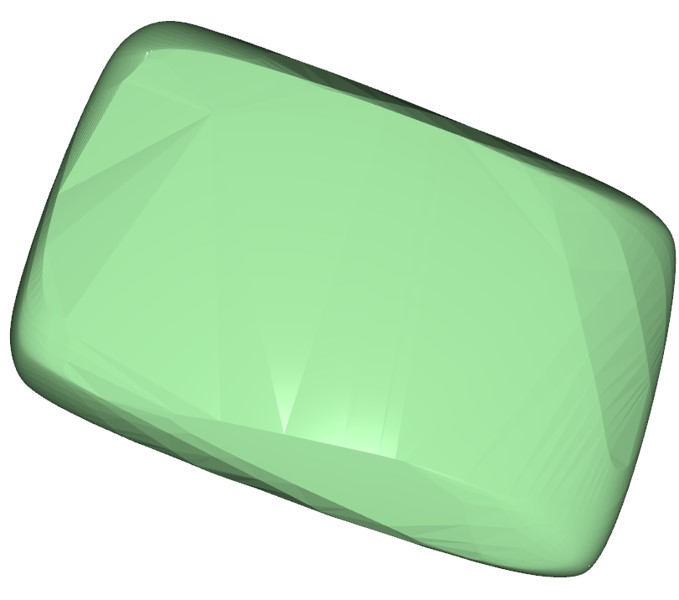}}
\caption{Images of the objects used in the collision detection benchmark}
\label{fig:ycbobject}
\end{figure}

\subsection{Additional details and results for peg-in-hole task} \label{subsec:appendixpeg}

\subsubsection{Cost landscape}

\begin{figure}[t]
\centering
\subfigure[Scenario1 (Left: after first touch, Right: after second touch)]{
\includegraphics[width=8.4cm]{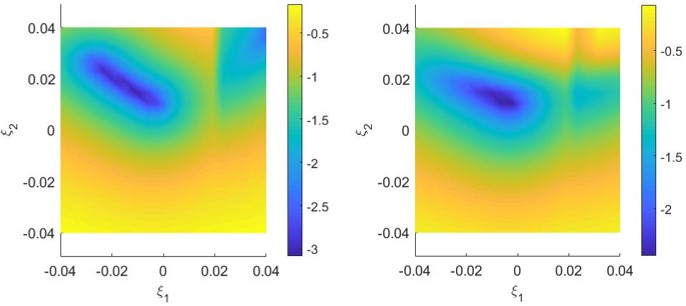}}
\subfigure[Scenario2 (Left: after first touch, Right: after second touch)]{
\includegraphics[width=8.4cm]{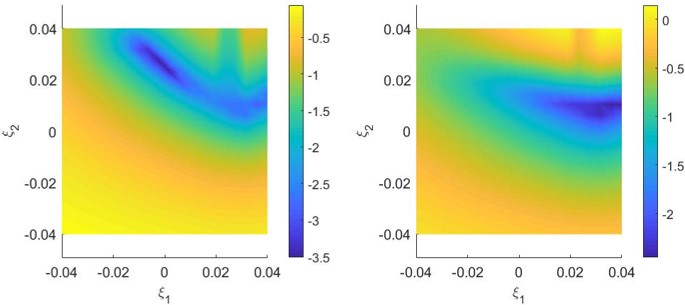}}
\caption{Images of the objects used in the collision detection benchmark}
\label{fig:costlandscape}
\end{figure}

To assess validity of the optimization-based formulation and the impact of multiple interactions, we visualize the cost landscape. 
For more intuitive interpretation, we assume that the uncertainty exists in the $x$ and $y$ positions of the hole, while the grasped peg pose are known (therefore, $\xi\in\mathbb{R}^2$). 
Also here, rectangular peg is employed.

Fig.~\ref{fig:costlandscape} illustrates the cost landscape obtained from our differentiable framework for the problem.
See also Fig.~\ref{fig:thumbnail} for the optimization path on the landscape.
As depicted, the solution initially obscured with multiple minima, becomes more apparent as interaction is added. This observation highlights the potential of our method in generating interesting results when integrated with active sensing. It suggests that incorporating additional interactions can enhance the identification and clarity of the optimal solution.

\subsubsection{Grasp parameterization}
\begin{figure}[t]
\centering
\subfigure{
\includegraphics[width=2.5cm]{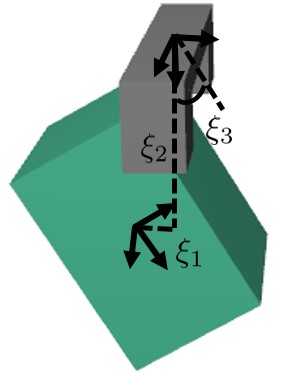}}
\caption{Parameterization of grasp pose for a rectangular peg. The parameterization is also similarly defined for a hexagonal peg.}
\label{fig:grasp_paramerterization}
\end{figure}

Fig.~\ref{fig:grasp_paramerterization} provides a visualization of how grasping is modeled. In the scenarios described in Sec.~\ref{subsec:peginhole}, the pose of the grasped peg can be effectively represented with just 3 parameters, offering intuitive understanding. However, for more intricate shapes of pegs and grippers, 6 parameters can be employed, while incorporating non-penetration constraints. Exploring scenarios that encompass these complexities would present an intriguing avenue for future research.

\subsubsection{Star-shaped geometry}

\begin{figure}[t]
\centering
    \centering
    \includegraphics[width=8.4cm]{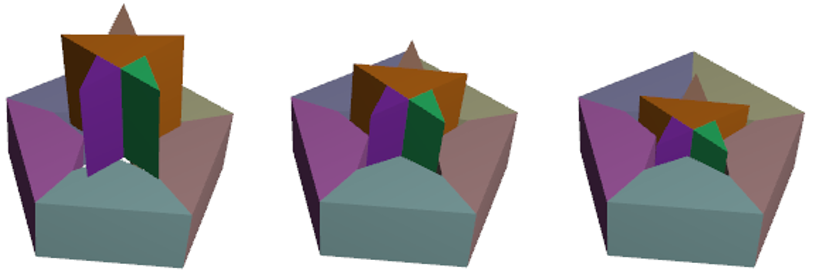}
    \caption{Snapshots of simulation results of star-shaped peg-in-hole manipulation using our uncertain pose estimation framework in online. Different colors are used to represent convex-decomposed shapes.}
    \label{fig:starpeg}
\end{figure}


To test our approach on more complex geometries, we implement a star-shaped peg-in-hole scenario. 
In this setup, both the peg and the hole are decomposed into five convex geometries, and a total of $25$ collisions are pre-defined. 
We validate the effectiveness of our method in successfully identifying and executing tasks in this scenario, as demonstrated in Fig.~\ref{fig:starpeg} and the supplementary video.

\end{document}